\documentclass{article} % For LaTeX2e
\usepackage[accepted]{icml2024}

\usepackage[hidelinks]{hyperref}
\usepackage{url}
\usepackage{amsmath, amssymb, amsthm}
\usepackage{xcolor}
\usepackage{graphicx}
\usepackage{enumerate}
\usepackage{multirow}
\usepackage{diagbox}
\usepackage{subcaption}
\newtheorem{theorem}{Theorem}

\icmltitlerunning{Submission to ICML 2024}

\begin{document}

\twocolumn[
\icmltitle{Enhancing Peak Assignment in \textsuperscript{13}C NMR Spectroscopy: A Novel Approach Using Multimodal Alignment
%Multi-View Molecular Graph Contrastive Learning \\with Multimodal Knowledge-Aware Metrics 
}



\icmlsetsymbol{equal}{*}

\begin{icmlauthorlist}
\icmlauthor{Hao Xu}{equal,yyy}
\icmlauthor{Zhengyang Zhou}{equal,yyy}
\icmlauthor{Pengyu Hong}{yyy}

\end{icmlauthorlist}

\icmlaffiliation{yyy}{Department of Computer Science, Brandeis University, Waltham, USA}

\icmlcorrespondingauthor{Pengyu Hong}{hongpeng@brandeis.edu}

\icmlkeywords{Machine Learning, ICML}
\vskip 0.3in
]


\printAffiliationsAndNotice{\icmlEqualContribution} % otherwise use the standard text.
%\title{Molecular Identification and Peak Assignment: Leveraging Multi-Level Multimodal Alignment on NMR}

% Authors must not appear in the submitted version. They should be hidden
% as long as the \iclrfinalcopy macro remains commented out below.
% Non-anonymous submissions will be rejected without review.

%\author{Hao Xu, Zhengyang Zhou \& Pengyu Hong 
%\thanks{ Use footnote for providing further information
% about author (webpage, alternative address)---
% \emph{not} for acknowledging funding agencies.  Funding acknowledgements go at the end of the paper.} \\
% Department of Computer Science\\
% Brandeis University\\https://www.overleaf.com/project/64e3de5ab844add4122bca73
% Waltham, MA, 02453, USA\\
% \texttt{\{haox,zhengyjo,hongpeng\}@brandeis.edu}
% }

% The \author macro works with any number of authors. There are two commands
% used to separate the names and addresses of multiple authors: \And and \AND.
%
% Using \And between authors leaves it to \LaTeX{} to determine where to break
% the lines. Using \AND forces a linebreak at that point. So, if \LaTeX{}
% puts 3 of 4 authors names on the first line, and the last on the second
% line, try using \AND instead of \And before the third author name.

\newcommand{\fix}{\marginpar{FIX}}
\newcommand{\new}{\marginpar{NEW}}
\newcommand{\specialcell}[2][t]{%
  \begin{tabular}[#1]{@{}c@{}}#2\end{tabular}}

% \neuripsfinalcopy % Uncomment for camera-ready version, but NOT for submission.
%\begin{document}

% \author{Hao Xu$^*$}
% \author{Zhengyang Zhou$^*$}
% \author{Pengyu Hong$^+$}
% \def\thefootnote{*}\footnotetext{These authors contributed equally to this work}\\
% \def\thefootnote{+}\footnotetext{Corresponding Author}\

% \author{%
%   Hao Xu$^{1,\dagger}$ \quad Zhengyang Zhou$^{1,\dagger}$ \quad Pengyu Hong$^{1,*}$  \\
%   $^1$ Department of Computer Science\\
%   Brandeis University\\
%   Waltham, MA 02453 \\
%   $^\dagger$ Equal Contribution
%   $\text{*}$ Corresponding Author\\
%   \texttt{\{haox,zhengyjo,hongpeng\}@brandeis.edu} \\
%   % examples of more authors
%   % \And
%   % Coauthor \\
%   % Affiliation \\
%   % Address \\
%   % \texttt{email} \\
%   % \AND
%   % Coauthor \\
%   % Affiliation \\
%   % Address \\
%   % \texttt{email} \\
%   % \And
%   % Coauthor \\
%   % Affiliation \\
%   % Address \\
%   % \texttt{email} \\
%   % \And
%   % Coauthor \\
%   % Affiliation \\
%   % Address \\
%   % \texttt{email} \\
% }

% \maketitle
\begin{abstract}
Nuclear magnetic resonance (NMR) spectroscopy plays an essential role in deciphering molecular structure and dynamic behaviors. While AI-enhanced NMR prediction models hold promise, challenges still persist in tasks such as molecular retrieval, isomer recognition, and peak assignment. In response, this paper introduces a novel solution, Knowledge-Guided Multi-Level Multimodal Alignment with Instance-Wise Discrimination (K-M\textsuperscript{3}AID), which establishes correspondences between two heterogeneous modalities: molecular graphs and NMR spectra. K-M\textsuperscript{3}AID employs a dual-coordinated contrastive learning architecture with three key modules: a graph-level alignment module, a node-level alignment module, and a communication channel. Notably, K-M\textsuperscript{3}AID introduces knowledge-guided instance-wise discrimination into contrastive learning within the node-level alignment module. In addition, K-M\textsuperscript{3}AID demonstrates that skills acquired during node-level alignment have a positive impact on graph-level alignment, acknowledging meta-learning as an inherent property. Empirical validation underscores the effectiveness of K-M\textsuperscript{3}AID in multiple zero-shot tasks.
\end{abstract}

\section{Introduction}
\label{sec1}
%\textcolor{brown}{First of all, we want to define two core concepts in this paper. We define \textbf{Reducible Substance(RS)}, such as molecules, as those can be decomposed into simpler elements. \textbf{Irreducible Element(IE)}, such as atom, is referred to the fundamental unit that cannot be broken down further without losing its essential properties or meaning.}

%\subsection{Multimodal meta-alignment}

%Multimodal Deep Learning (MMDL) emerges as a dynamic and interdisciplinary research domain focused on adapting artificial intelligence to extract, represent, connect and integrate heterogeneous data from multiple sources (such as text, images, audio, video, sensor data, etc)  \cite{liang2022foundations, jabeen2023review}). 

Nuclear magnetic resonance (NMR) spectroscopy has found broad applications in various scientific domains, such as chemistry, environmental science, food science, material science, and pharmaceuticals, by providing insights into molecular dynamics and structures \cite{gunther1994nmr, claridge2016high, yu2021recent}. The details of NMR spectra can be influenced by through-bond and through-space interactions, serving as "fingerprints" to deduce atomic connectivity, relative stereochemistry, and conformations. The conventional approach for elucidating molecular structures and attributing peaks has long relied on manual determination by organic chemists \cite{guan2021real}. However, the interpretation of NMR spectra is not straightforward, particularly when dealing with isomers and complex molecules consisting of multiple stereogenic (chiral) centers \cite{wu2023elucidating, chhetri2018recent}. Even an expert chemist may encounter significant difficulties in accurately assigning isomeric compounds with extremely similar NMR spectra due to this complexity \cite{nicolaou2005chasing}.
%Given the rising advancements of artificial intelligence (AI) in many scientific domains \cite{ karniadakis2021physics, artrith2021best, sapoval2022current}), the incorporation of AI has been anticipated to streamline and optimize the direct interpretation of NMR spectra.

\begin{figure*}[ht!]
    \includegraphics[width=1\textwidth]{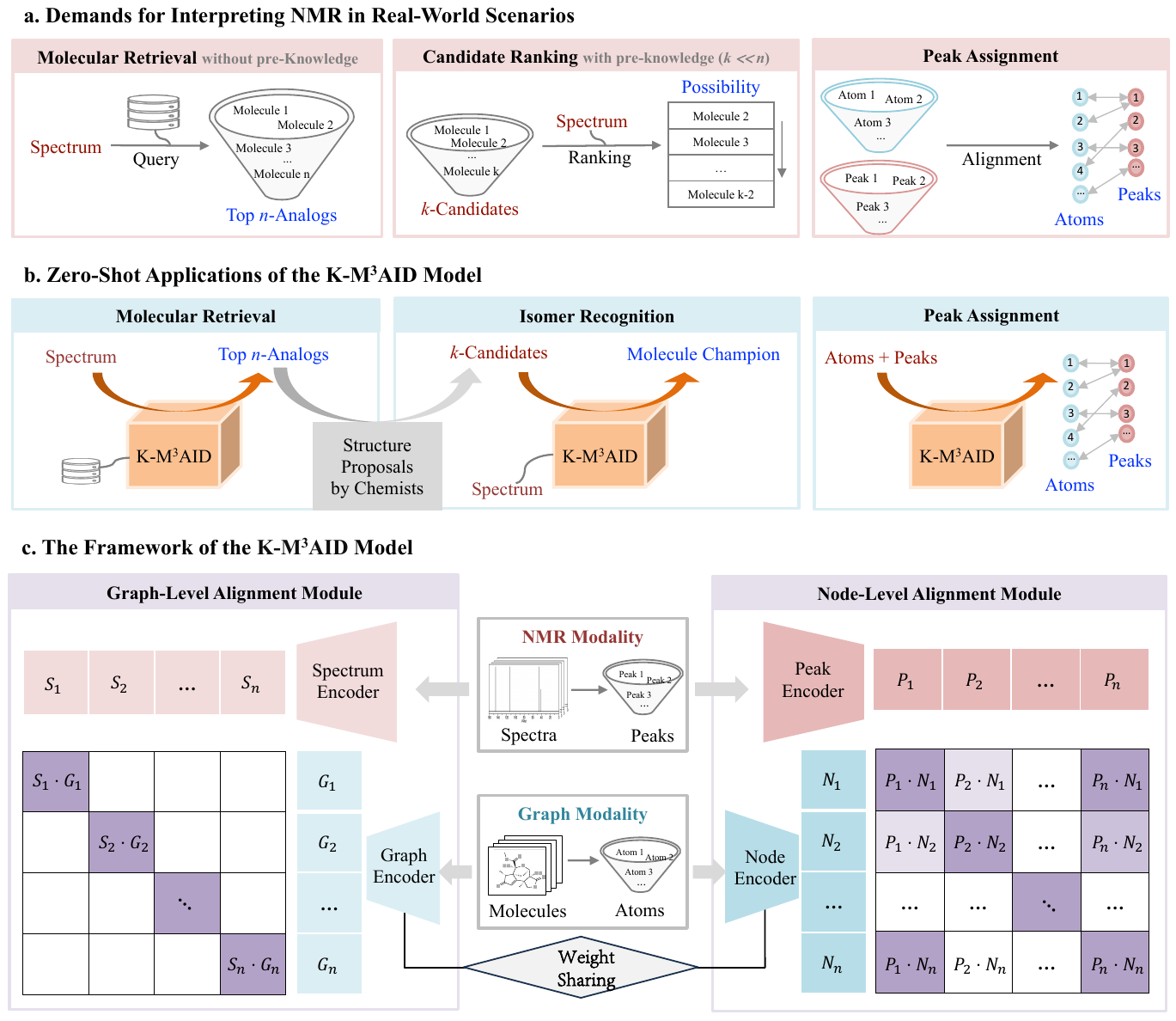}
    \caption{a. Demands for interpreting NMR spectra in real-world scenarios: molecular retrieval, candidate ranking, and peak assignment; b. Zero-shot applications of the K-M\textsuperscript{3}AID model: molecular retrieval, isomer recognition, and peak assignment; c. The framework of K-M\textsuperscript{3}AID model: the molecular alignment module is responsible for optimizing the the correspondence between modalities at the molecular level, the atomic alignment module focus on the fine-tuning of atomic positioning on the spectrum, and the communication channel dynamically adjusts the flow of gradients between node encoder and graph encoder during the training process. $S$ for spectrum embedding, $G$ for graph embedding, $P$ for peak embedding and $N$ for node embedding.}
    \label{fig:CLMA_Architecture}
\end{figure*}

% the developed ML apparoch in NMR
While recent AI-enhanced NMR spectrum prediction models show promise in generating spectra from candidate structures \cite{chen2020review, jonas2022prediction, kuhn2022applications}, these models still face challenges in peak assignment due to their high error tolerance and a lack of precise point-to-point guidance. Since peak assignment is a determining step in isomer recognition, these models fall short in achieving accurate isomer recognition. Another contributing factor is the absence of quantitative ranking for candidate isomers in their implementation. In addition, the success of these models requires a good level of prior knowledge of molecular structures to construct candidates. However, real-world practice often demands spectral interpretation before detailed structural information is available. For instance, when identifying an unknown compound from a plant, there is limited or no knowledge of this compound. Thus, the interpretation of spectra should transition from spectral data to structural elucidation. Therefore, it is imperative to utilize advanced AI methodologies to simplify NMR spectral interpretation, particularly in tasks such as molecular retrieval, candidate ranking, and peak assignment (see Figure ~\ref{fig:CLMA_Architecture}.a).

% Multimodal Alignment
In the realm of data representations for NMR interpretation, two heterogeneous modalities come into play: NMR spectrum and molecular graph. A NMR spectrum is a sequence-based chemical modality that captures molecular structural and electronic details through an NMR spectrometer, translating such information into NMR peaks. A molecular graph encapsulates molecular structural and electronic information through the arrangement of nodes and edges, along with their respective attributes. The analysis of molecular structure and peak assignment requires clear correspondence across these two heterogeneous modalities, which can be formulated as a multimodal alignment problem.%However, multimodal alignment commonly encounters challenges in addressing data heterogeneity, semantic gaps among modalities, collecting paired data, designing effective cross-modal attention mechanisms, and formulating appropriate evaluation metrics \cite{liang2023foundations, jabeen2023review}.

Molecules are distinguished by the distinctive configuration of atoms coupled with bonding patterns, giving rise to distinct spectra. As molecular diversity is extensive, it is impractical to include all molecules and their spectra in a training dataset. However, the corresponding atomic diversity is comparatively constrained. In the context of the multi-view nature of molecules, it is a sound approach to analyze molecular structures by interpreting spectra at the atomic level. Thus, this task can be formulated as a meta-learning problem, which is a branch of metacognition concerned with understanding one's own learning and learning processes.
%However, meta-learning typically concerns meta-overfitting, distance evaluation metrics, model interpretability, transferability, and generalization \cite{huisman2021survey, feng2022meta, ma2022multimodality, vettoruzzo2023advances}.

In light of these challenges and opportunities, we propose a novel framework, K-M\textsuperscript{3}AID (Knowledge-Guided Multi-Level Multimodal Alignment with Instance-Wise Discrimination), aiming to achieve reliable analog retrieval, candidate ranking, and peak assignment in the interpretation of NMR spectra (see Figure ~\ref{fig:CLMA_Architecture}.b). The overview of our K-M\textsuperscript{3}AID framework features a dual-coordinated contrastive learning architecture, comprising three key components: a graph-level alignment module, a node-level alignment module, and a communication channel.
The graph-level alignment module establishes correspondences between molecules and their individual $^{13}$C NMR spectra. Given that each unique molecule produces a distinct spectral signature, this module employs a straightforward cross-entropy loss for effective contrastive learning.
The node-level alignment module aligns each Carbon atom within the molecules with their signal peaks on the spectrum. Unlike the diverse and distinctive molecular spectral signatures, many atoms exhibit chemical symmetry and magnetic equivalence within the same molecule, corresponding to the same peaks. However, atoms with different local surroundings can still present significant similarity on the spectrum, introducing a heightened level of complexity. To address these complex scenarios, we introduce knowledge-guided instance-wise discrimination based on contrastive learning in the node-level alignment module (see Figure \ref{fig:knowledge-span}).
The communication channel dynamically adjusts the flow of gradients between the node encoder and the graph encoder from two modules during the training process.

In summary, our contribution encompasses three significant aspects: \textbf{\textit{Conceptually:}} We integrate cross-modal alignment at two architectural levels, namely graph and node levels, within the K-M\textsuperscript{3}AID framework. This integration facilitates rapid adaptation, significantly boosting the efficiency of learning for zero-shot tasks. \textbf{\textit{Methodologically:}} We introduce knowledge-guided instance-wise discrimination for cross-modal contrastive learning, leveraging continuous and domain-specific features with inherent natural ordering. To the best of our knowledge, this is the first demonstration of knowledge-guided instance-wise discrimination-based cross-modal contrastive learning, transforming discrete comparisons into a continuous paradigm. \textbf{\textit{Empirically:}} We substantiate the effectiveness of K-M\textsuperscript{3}AID through its successful application to various zero-shot tasks, including molecular retrieval, isomer recognition, and peak assignment.

\section{Preliminaries}

\label{multimodal-alignment}
\textbf{Multimodal Alignment:}
Multimodal alignment, as defined in the literature \cite{baltrusaitis2017multimodal}, involves establishing relationships and correspondences among sub-components of instances from two or more modalities. A typical example is identifying specific regions in an image that correspond to words or phrases in a given caption \cite{karpathy2015deep}. This approach offers numerous benefits, including enhanced data interpretation, heightened accuracy and robustness, overcoming limitations of single-modal systems, and better addressing real-world complexity \cite{baltrusaitis2017multimodal}, \cite{summaira2021recent}, \cite{akkus2023multimodal}. CLIP (Contrastive Language-Image Pretraining) \cite{radford2021learning, li2021supervision} is one of the most widely adopted frameworks for multimodal alignment. As highlighted in the introduction, molecular information originates from diverse sources such as molecule graphs and NMR spectroscopy. To leverage effective alignments of this multifaceted information across different modalities, we adopt the CLIP framework with graph neural networks (GNN) \cite{xu2018powerful}, \cite{wu2022graph} to encode molecular information and neural network encoders \cite{serra2018towards} to encapsulate NMR information.

%Multimodal alignment is defined as finding relationships and correspondences between sub-components of instances from two or more modalities \cite{baltrusaitis2017multimodal}). One of the typical examples is to find the areas of the image corresponding to the caption’s words or phrases given an image and a caption \cite{karpathy2015deep}). Multimodal alignments can provide us a lot benefits including richer data interpretation, high accuracy and robustness, overcoming limitations of Single-Modal Systems and better approaching the Real-World Complexity \cite{baltrusaitis2017multimodal},\cite{summaira2021recent},\cite{akkus2023multimodal}. As mentioned in introduction, the information of molecules comes from different sources such as molecule graph and NMR spectroscopy. In order to integrate and process the above information in a fine way, we need to apply multimodal learning.  In the scope of our tasks, we use GNN \cite{xu2018powerful,wu2022graph}) to represent the corresponding molecular and atom information. In addition, we use Neural Network Encoder \cite{serra2018towards} to represent the NMR information.

\label{meta-learning}
% \textcolor{red}{Cite the orignal definition of meta-learning. The first meta learning example is ... in [which paper]. In the field of chemnistry, .....}
\textbf{Meta-Learning:} 
Meta-learning is defined as the process of learning how to learn across tasks \cite{vilalta2002meta}. More specifically, it leverages skills previously acquired from related tasks to the current one \cite{lake2017building}. With more skills learned, acquiring new ones becomes easier, requiring fewer examples and less trial-and-error \cite{vanschoren2018meta,finn2017metamodel}. A meta-learner is trained on a diverse set of object recognition tasks. During this training, it learns common features, patterns, and strategies for recognizing objects. Once trained, when presented with a new, previously unseen object category, the meta-learner can rapidly adapt and achieve high recognition accuracy, leveraging the knowledge acquired from the diverse training tasks to perform in this novel recognition task \cite{finn2017metamodel}. A profound understanding of atom properties allows us to extend our vision to previously unseen molecules, aligning with the principles of meta-learning in artificial intelligence.

%Meta-learning is defined as to learn how to learn across task \cite{vilalta2002meta}). More specifically, we start from skills learned earlier in related tasks, reuse approaches that worked well before, and focus on what is likely worth trying based on experience \cite{lake2017building}.With every skill learned, learning new skills becomes easier, requiring fewer examples and less trial-and-error \cite{vanschoren2018meta,finn2017metamodel}). A meta-learner is trained on a diverse set of object recognition tasks. During this training, it learns common features, patterns, and strategies for recognizing objects. Once trained, when presented with a new, previously unseen object category, the meta-learner can rapidly adapt and achieve high recognition accuracy, leveraging the knowledge it acquired from the diverse training tasks to perform in this novel recognition task \cite{finn2017metamodel}. For instance, in the field of chemistry, molecules are composed of atoms. If we can study the properties of atoms well enough, we can expand the vision of unseen molecules, which corresponds to the category of meta-learning in artificial intelligence.

% \textcolor{red}{The basic concepts  to instance-wise discrimination. This is the technique to combine the previous two method. After acquiring molecule/atom representations from different sources, we need methodology to tell whether two representation are for the same or similar molecules/atoms. In this case, contrastive learning, another technique in machine learning comes into play.}
\label{contrastive-learning}
\textbf{Contrastive Learning:} 
Contrastive learning focuses on discerning similarities and differences between items \cite{le2020contrastive, jaiswal2021contraself, liu2021selfcontrastive}. A fundamental aspect of this process involves instance-wise discrimination \cite{wu2018unsupervised}. Models incorporating instance-wise discrimination not only foster an understanding of the inherent data structure but also enhance generalization capabilities. This is attributed to the contrastive learning approach, which prioritizes distinguishing between instances rather than memorizing specific labeled examples. In chemistry, each molecule/atom is treated as a distinct instance, and the learning algorithm focuses on distinguishing each molecule/atom based on its context.

%Contrastive learning is an important technique in machine learning that focus on learning how to distinguish between similar and dissimilar items \cite{le2020contrastive,jaiswal2021contraself,liu2021selfcontrastive}. One of the key elements to conduct such similarity comparison is by instance-wise discrimination \cite{wu2018unsupervised}). Those models with instance-wise discrimination not only develops an understanding of the underlying structure of the data, but also improve the capability of generalization. That is due to contrastive learning focus on distinguish classes instead of memorizing specific labeled examples. In the field of chemistry, each molecule/atom is treated as a distinct class. The learning algorithm focuses on distinguishing each molecule/atom based on its content. CLIP (Contrastive Language-Image Pretraining) \cite{radford2021learning,li2021supervision}) is one of the most popular in this category.  In the scope of our task, we will extend/improve the metric for measuring similarities in the contrastive learning framework.

\section{Our Method}
\label{sec3}

% In this section, we first introduce the architecture and mechanism of CLMA framework, an end-to-end system designed for multimodal meta-alignment. Our primary emphasis lies in the precise alignment of individual atoms within molecular graphs to their corresponding peaks on the chemical spectrum. Next, we present the intricacies of Knowledge-Guided Instance-Wise Discrimination in the context of our meta-alignment task.

In this section, we firstly introduce Knowledge-Guided Instance-Wise Discrimination. Then, we present the architecture of the K-M\textsuperscript{3}AID framework, an end-to-end system designed for multi-level multimodal alignment, along with its loss function.

\subsection{Knowledge-Guided Instance-Wise Discrimination Contrastive Learning} 

\begin{figure*}[ht]
    \centering
    \includegraphics[width=0.95\textwidth]{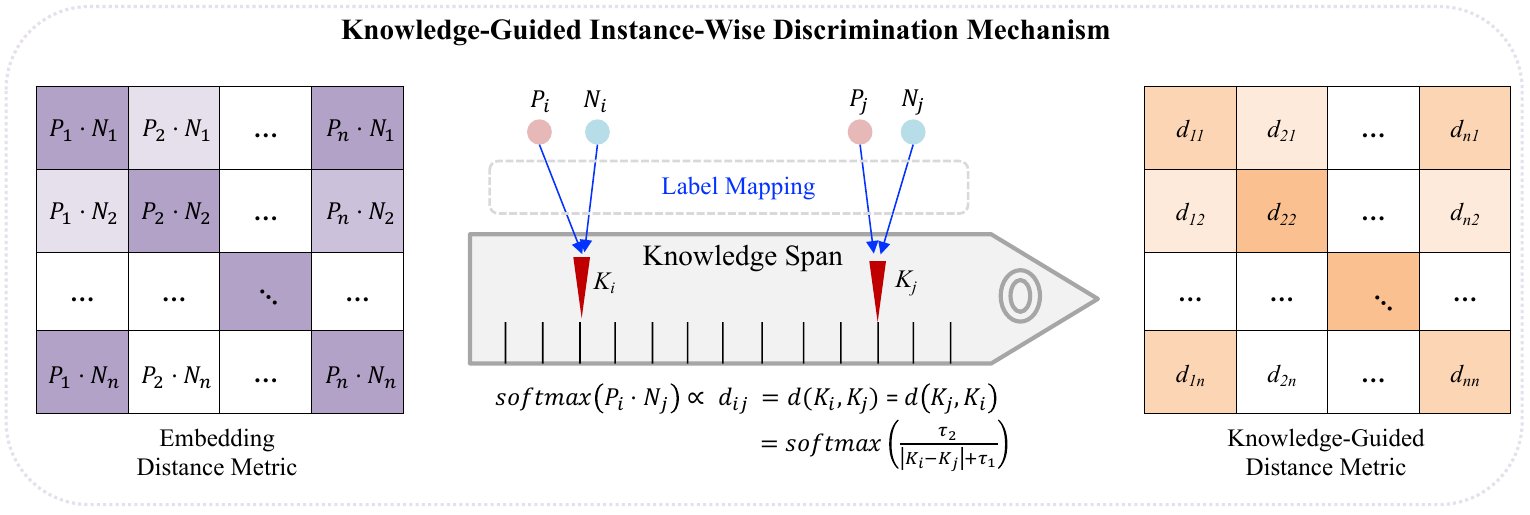}
    \caption{Knowledge-Guided Instance-Wise Discrimination Mechanism. $K_{i}$ and $ K_{j}$ represent the corresponding knowledge span labels for $i^{th} $ and $ j^{th}$ items.}
    \label{fig:knowledge-span}
\end{figure*}
 
Knowledge Span, which we define as a continuous and domain-specific feature, exhibits natural ordering and is able to offer guidance for contrastive learning. As such, we introduce a novel approach into contrastive learning, termed Knowledge-Guided Instance-Wise Discrimination (see Figure ~\ref{fig:knowledge-span}). This approach expands the scope of contrastive learning from confined comparisons (pre-determined negative and positive pairs) to unrestricted comparisons (no need for pre-determination). This extension removes the necessity of explicitly defining such pairs, thus mitigating the potential introduction of human bias.

% Here, we introduce the knowledge-guided instance-wise discrimination contrastive learning loss for meta-alignment, denoted as $CL_{I}$, in a broader context.

Suppose $\mathcal{M}$ is the set of instances. $\mathcal{A} \subset \mathbb{R}^{d_1}$ is the set of tunable instances' embeddings in modality A, $\mathcal{B} \subset \mathbb{R}^{d_1}$ is the set of tunable instances' embeddings in modality B, and $\mathcal{K} \subset \mathbb{R}^{d_2}$ is the corresponding fixed knowledge span label that can guide the relative distance learning between components in $\mathcal{A}$ and $\mathcal{B}$. Thus, the size of $\mathcal{A}$, $\mathcal{B}$, $\mathcal{K}$ are $|\mathcal{M}|$, respectively.

Let $\mathcal{A}_i$ be the $i^{th}$ instance embedding of $\mathcal{A}$, and $\mathcal{B}_j$ be the $j^{th}$ instance embedding of $\mathcal{B}$. We define the distance function between $\mathcal{A}_i$ and $\mathcal{B}_j$ as $d_{E}(\mathcal{A}_i, \mathcal{B}_j) = \mathcal{A}_i \cdot \mathcal{B}_j \rightarrow \mathbb{R}^{+}$, and calibration function $d(\mathcal{K}_{i}, \mathcal{K}_{j}) \rightarrow \mathbb{R}^{+}$ with a monotonic property and constraint $\sum_{j=1}^{|\mathcal{M}|}d(\mathcal{K}_{i}, \mathcal{K}_{j}) = 1$, in which $\mathcal{K}_{i}$ and $\mathcal{K}_{j}$ serve as the designated Knowledge Span Label. We introduce the Knowledge Span Guided Loss (KSGL) as follows:
\begin{align}
     KSGL(i) &=-\displaystyle\sum_{\substack{1\leq j \leq |\mathcal{M}|}} d(\mathcal{K}_{i}, \mathcal{K}_{j}) \log \frac{e^{d_{E}(\mathcal{A}_{i}, \mathcal{B}_{j})}}{\displaystyle\sum_{\substack{1\leq k \leq |\mathcal{M}|}} e^{d_{E}(\mathcal{A}_{i}, \mathcal{B}_{k})}} \\
    & = -\displaystyle\sum_{\substack{1\leq j \leq |\mathcal{M}|}} d(\mathcal{K}_{i}, \mathcal{K}_{j}) \log (\text{softmax}(d_{E}(\mathcal{A}_{i}, \mathcal{B}_{j})))
    \label{equ:ksgl}
\end{align}

In particular, when it reaches ideal optimum, $d(\mathcal{K}_{i}, \mathcal{K}_{j})$ and $d_{E}(\mathcal{A}_{i}, \mathcal{B}_{j})$ reaches the following relation:
\begin{equation}
    d(\mathcal{K}_{i}, \mathcal{K}_{j}) = \text{softmax}(d_{E}(\mathcal{A}_{i}, \mathcal{B}_{j}))
    \label{equ:knoledge-stable-relation}
\end{equation}
For detail proof, please refer to Appendix~\ref{appendix:knowledge-span-guide-proof}.
As a result, the corresponding $CL_{instance}$ is expressed as following:
\begin{align}
    CL_{instance} &= \frac{1}{|\mathcal{M}|}\displaystyle\sum_{1 \leq i \leq |\mathcal{M}|}KSGL(i)
    \label{eq:ie-loss}
\end{align}
%\textcolor{brown}{For $CLS_{IE}$, the form is similar to $CL_{IE}$, except the corresponding input is $A_{i}^{T}$ and $B_{j}^{T}$ for each $CLS_{IE}(i)$.}

\subsection{Architecture \& Contrastive Learning Loss} 

The K-M\textsuperscript{3}AID framework is a dual-CLIP architecture (see Figure \ref{fig:CLMA_Architecture}), comprising three critical components: a graph-level alignment module, a node-level alignment module, and a communication channel. The graph-level alignment module adopts a gradient-asymmetric CLIP mechanism. While two unimodal encoders work in conjunction, only the from-scratch graph encoder (GIN, \cite{xu2018powerful}) undergoes dynamic training throughout the process; the pre-trained spectrum encoder \cite{yang2021cross} remains fixed. Both encoders are complemented by dedicated projection layers, facilitating the mapping of embeddings into a joint space. The node-level alignment module adopts a gradient-symmetric CLIP mechanism. It is equipped with two from-scratch unimodal encoders, the node encoder and the peak encoder, as well as their dedicated projection layers. The graph encoder in the graph-level alignment module shares part of the weights with the node encoder in the node-level alignment module, serving as the communication channel.

The synergy between these two modules is pivotal, collectively contributing to the loss function, expressed as
\begin{equation}
    L = CL_{graph} + CL_{node},
\end{equation}
where $CL_{graph}$ represents the contrastive learning loss in the graph-level alignment module by Equation~\ref{eq:loss_function}, and  $CL_{node}$ represents the contrastive learning loss in the node-level alignment module by Equation~\ref{eq:node-loss-function}. %\textcolor{brown}{$CLS_{RS}$ and $CLS_{IE}$ are the corresponding symmetric terms for $CL_{RS}$ and $CL_{IE}$. }

Let $i$ denote the $i^{th}$ instance, and $j$ denote the $j^{th}$ instance. Then $x_{i}$ denotes the raw input in modality A for the $i^{th}$ instance and $y_{j}$ denotes the raw input in modality B for the $j^{th}$ instance. Suppose $f_{x}\left( \cdot \right)$ represent the encoding function for modality A, and $f_{y}\left( \cdot \right)$ denote the encoding function for modality B. In graph-level alignment module, these two encoding functions,  should map $x_{i}$  and $y_{j}$   to a proximate location in the joint embedding (inter-modality) if $i = j$.
\begin{align}
    CL_{graph}(i) &= - \log \frac{e^{\delta(x_{i}, y_{i})}}{\displaystyle\sum_{1 \leq j \leq N} e^{\delta(x_{i}, y_{j})}} \\
    &= -\text{log}(\text{softmax}(\delta(x_{i}, y_{i}))
    \label{eq:loss_function}
\end{align}
Where $\delta(x_{i}, y_{j}) = \left( f_{x}(x_{i})^{T} \cdot f_{y}(y_{j}) \right)$, $N$ is the total number of instances from the current batch. 

Thus, the total $CL_{graph}$ is expressed as following:
\begin{align}
    CL_{graph} = \frac{1}{N}\displaystyle\sum_{1 \leq i \leq N}CL_{graph}(i)
    \label{eq:rs-loss}
\end{align}
This design for the loss aims to match the same instance cross different modalities.

\section{Experiments}
\label{sec4}
To thoroughly evaluate the performance of K-M\textsuperscript{3}AID, we compare it with other baselines across various zero-shot downstream tasks, including molecular retrieval, isomer recognition, and peak assignment. Please refer to the detailed settings of pre-training and downstream tasks in Appendix \ref{appendix:exp-setting}.

\subsection{Chosen Knowledge Span-ppm}
\label{Knowledge-Span-ppm}
$^{13}$C NMR uncovers molecular structures by providing the chemical environments of carbon atoms and their magnetic responses to external fields, quantifying these features in parts per million (ppm) relative to a reference compound like tetramethylsilane (TMS), simplifying comparisons across experiments. Thus, continuous peak positions, measured in ppm, can serve as a robust knowledge span to facilitate instance-wise discrimination for this contrastive learning task.

For the node-level alignment module, $\mathcal{A}$ is the set of node embeddings for Carbon atoms in the molecular graph modality, and $\mathcal{B}$ is the set of peak embeddings for respective Carbon atoms in the NMR modality. $\mathcal{K}$ is the set of ppm values for each corresponding Carbon atom in $\mathcal{A}$ and $\mathcal{B}$. Suppose $S_{i}$ is the peak for the $i^{th}$ Carbon Atom, and $S_{j}$ is the  $j^{th}$ peak. $d(\cdot, \cdot)$ is then defined as follows:
\begin{align}
d(\mathcal{K}_{i}, \mathcal{K}_{j}) &= d(S_{i},S_{j})\\
& = softmax(\frac{\tau_{2}}{|S_{i} - S_{j}|+\tau_{1}})
\label{equ;ppm-guide-zhou}
\end{align}

where $\tau_{1}$ and $\tau_{2}$ are temperature hyper-parameter. For further discussion of selection about $\tau_{1}$ and $\tau_{2}$, please refer to Appendix~\ref{appendix:tau-ablation-study}. Then, the final form of contrastive loss for node-level alignment according to Equation \ref{equ:ksgl} and Equation \ref{eq:ie-loss} is as following: 
% \begin{align}
% CL_{node} &= -\frac{1}{|\mathcal{M}|}\sum_{\substack{1\leq j \leq |\mathcal{M}|}} d(K_{i},K_{j}) \cdot \\ & \log \frac{e^{d_{E}(\mathcal{A}_{i}, \mathcal{B}_{j})}}{\sum_{\substack{1\leq k \leq |\mathcal{M}|}} e^{d_{E}(\mathcal{A}_{i}, \mathcal{B}_{k})}}
% \label{eq:node-loss-function}
% \end{align}
\begin{align}
CL_{node} &= -\frac{1}{|\mathcal{M}|}\sum_{j=1}^{|\mathcal{M}|} d(\mathcal K_{i}, \mathcal K_{j})\cdot log \frac{e^{d_{E}(\mathcal{A}_{i}, \mathcal{B}_{j})}}{\sum_{k=1}^{|\mathcal{M}|} e^{d_{E}(\mathcal{A}_{i}, \mathcal{B}_{k})}}
\label{eq:node-loss-function}
\end{align}
Here, $i$ and $j$ are indices of atoms. $\mathcal{A}_{i} \in \mathcal{A}$ represents the embedding of $i-th$ atom in modality A while $\mathcal{B}_{i} \in \mathcal{B}$ represents the embedding of $i-th$ atom in modality B, $\mathcal{K}_{j} \in \mathcal{K}$ represents peaks, and $th$ is the abbreviation for the threshold.

\subsection{Baselines}
\textbf{No Communication:} In contrast to the communicative mechanism of K-M\textsuperscript{3}AID, one of the baselines is established without the utilization of a communication channel (denoted as No Comm.).

\textbf{Strong-Pair-based Instance-Wise Discrimination:} We explore an alternative baseline where the knowledge-guided instance-wise discrimination in the node-level alignment module is replaced with strong-pair-based instance-wise discrimination (denoted as SP). SP enforces a precise match in node-level (atom-peak) alignment, ensuring that only correct pairs established during the training process are considered. The mathematical definition of a strong pair is as follows:
\begin{equation}
    \textit{Strong Pair: } |S_{i} - S_{j}| = 0,
\end{equation}
where $i$, $j$ represent the indices of peaks.

\textbf{Weak-Pair-based Instance-Wise Discrimination:} We replace the knowledge-guided instance-wise discrimination with weak-pair-based instance-wise discrimination (denoted as WP) in the node-level alignment module. WP broadens the matching criteria of SP, allowing for multiple matches within a specified threshold set for the distance of their corresponding parts per million (ppm, referenced in Section ~\ref{Knowledge-Span-ppm}). The mathematical definition of a weak pair is as follows:
\begin{equation}
    \textit{Weak Pair: } |S_{i} - S_{j}| \le th,
\end{equation}
where $i$, $j$ represent the indices of peaks.

\subsection{Results}
\subsubsection{Validation Performance}
The K-M\textsuperscript{3}AID model showcases an impressive validation accuracy of 95.5\% in aligning molecules with spectra within the graph-level alignment module. In direct comparison, K-M\textsuperscript{3}AID outperforms alternative models such as SP, WP, and the model without a communication mechanism (No Comm.), demonstrating superior performance with a margin ranging from approximately 1\% to 6\% in the graph-level alignment module (refer to Table~\ref{tab:knowledge-threshold-comparison}). Notably, SP significantly outperforms WP, and as the matching criteria threshold widens, the performance of the latter deteriorates.

In the context of peak-atom alignment within the node-level alignment module, K-M\textsuperscript{3}AID and the model without a communication mechanism exhibit comparable accuracies (refer to Table~\ref{tab:knowledge-threshold-comparison}). However, K-M\textsuperscript{3}AID showcases slightly better stability across 5-fold cross-validation. Moreover, K-M\textsuperscript{3}AID demonstrates superiority in peak-atom alignment when compared to SP and WP. This superiority may arise from the inherent limitations of both strong and weak pair definitions, which fail to precisely calibrate the diverse relationships among the elements. This finding is further supported by the significant decreases in the accuracy of peak-atom alignment as the threshold of weak pair increases.

\begin{table*}[ht]
    \caption{Batch-wise validation accuracy (\%) of K-M\textsuperscript{3}AID and baselines with $epochs = 200$. For WP, the threshold is configured at 1, 5, and 10 ppm.
    %{Note: SP represents strong-pair-based instance-wise discrimination; WP represents week-pair-based instance-wise discrimination; No Comm. represents the model without communication channel.}
    }
    \label{tab:threshold-sensitivity}
    \small
    \centering
    \begin{tabular}{cccccccc}
        \hline
        Alignment & SP & WP(\textit{th}=1) & WP(\textit{th}=5) & WP(\textit{th}=10) & No Comm. & K-M\textsuperscript{3}AID \\
        \hline
        Graph-Level & 93.5$\pm$0.6 & 91.3$\pm$0.8 & 90.3$\pm$0.6 & 88.4$\pm$1.4 & 94.6$\pm$0.4 & \textbf{95.5$\pm$0.4} \\
        Node-Level & 89.3$\pm$0.4 & 83.7$\pm$0.6  & 79.8$\pm$0.5 & 66.1$\pm$2.5 & 90.4$\pm$0.2  & \textbf{90.3$\pm$0.1} \\
        \hline
    \end{tabular}
    \vspace{-3pt} 
    \label{tab:knowledge-threshold-comparison}
\end{table*}

% \begin{table}[ht]
%     \caption{Validation accuracy of SP-based and WP-based models with $epochs = 200$.}
%     \label{tab:threshold-sensitivity}
%     \small
%     \centering
%     \begin{tabular}{lccccccc}
%         \hline
%         Method & SP & WP(th=1) & WP(th=2) & WP(th=5) & WP(th=10) & K-M\textsuperscript{3}AID \\
%         \hline
%         RS MMA & 93.5$\pm$0.6 & 91.3$\pm$0.8 & 90.6$\pm$0.4 & 90.3$\pm$0.6 & 88.4$\pm$1.4 & \textbf{94.7$\pm$0.4} \\
%         IE MMA & 89.3$\pm$0.4 & 83.7$\pm$0.6 & 83.2$\pm$0.2 & 79.8$\pm$0.5 & 66.1$\pm$2.5 & \textbf{90.2$\pm$0.1} \\
%         \hline
%     \end{tabular}
%     \label{tab:knowledge-threshold-comparison}
% \end{table}

\subsubsection {Performance on Zero-Shot Molecular Retrieval}

We conduct a systematic evaluation of the effectiveness of our K-M\textsuperscript{3}AID model, comparing it with baseline models in the zero-shot molecular retrieval task across datasets of varying magnitudes. Detailed results are presented in Table~\ref{tab:pubmed-retrievel}. The K-M\textsuperscript{3}AID model consistently attains an impressive top-1 accuracy of approximately 95.8\% in molecular retrieval when the molecular reference library comprises 100 entries. This performance surpasses that of alternative mechanisms such as SP (95.3\%), WP (92.9\%), and No Comm. (94.8\%). As the molecular reference library expands to 1000 entries, the K-M\textsuperscript{3}AID model exhibits notable superiority, achieving accuracy levels of 80.4\%, 1.8\%, 8.7\%, and 2.8\% higher than SP, WP, and No Comm. mechanisms, respectively. The advantage of K-M\textsuperscript{3}AID becomes even more pronounced when the library size reaches 10,000 entries. In this scenario, K-M\textsuperscript{3}AID yields 46.3\% at top-1 accuracy, showcasing advancements of 10.5\%, 13.6\%, and 6.2\% over SP, WP, and No Comm. mechanisms, respectively. Even with larger molecular reference libraries, such as 100,000 and 1,000,000 entries, K-M\textsuperscript{3}AID consistently outshines SP, WP, and No Comm. mechanisms. These compelling results distinguish the K-M\textsuperscript{3}AID model as an exceptional choice in scenarios demanding robust performance in molecular retrieval tasks.

\setlength{\tabcolsep}{1.5pt}
\begin{table}[!ht]
    \centering
    \small
    \caption{Zero-shot molecular retrieval at top 1 accuracy across datasets of varying sizes. (For more statistics, such as top 5, top 10 and top 25, please refer to Appendix Table ~\ref{tab:pubmed-retrievel-th-whole}. For similarity comparison between the molecules and the Top 1 neighbor by different retrieval methods, please consult Appendix Table ~\ref{tab:pubmed-retrievel-th-whole-similarity}.)}
        \begin{tabular}{cccccc}
        \hline      
        \multicolumn{1}{c}{Method}      
        & \multicolumn{1}{c}{$10^2$} & \multicolumn{1}{c}{$10^3$} & \multicolumn{1}{c}{$10^4$} & \multicolumn{1}{c}{$10^5$} & \multicolumn{1}{c}{$10^6$} \\ \hline
        \multirow{1}{*}{K-M\textsuperscript{3}AID} %& Top 1(\%) 
         & 95.8$\pm$1.0 & 80.4$\pm$3.9 & 46.3$\pm$1.2 & 18.0$\pm$0.8 & 5.8$\pm$1.7 \\
                                        % \hline
        \multirow{1}{*}{SP}                   %& 99.7$\pm$0.22                    
        & 95.3$\pm$0.8                    & 78.6$\pm$2.7                      & 35.8$\pm$3.8                       & 12.9$\pm$1.6                        & 3.4$\pm$0.9                          \\
                                         %\hline
                \multirow{1}{*}{WP (\textit{th} = 1)} &  92.9$\pm$0.6                     & 71.7$\pm$1.0                      & 32.7$\pm$1.3                       & 10.7$\pm$0.5                        & 3.6$\pm$0.7                          \\
                                         %\hline
                 \multirow{1}{*}{No Comm.} 
        & 94.8$\pm$1.2                     & 77.6$\pm$1.4                      & 40.1$\pm$1.2                       & 14.4$\pm$0.9                        & 4.1$\pm$1.1                          \\
                                         \hline
        \end{tabular}
        \label{tab:pubmed-retrievel}
        \vspace{2pt} 
\end{table}

\subsubsection {Performance on Zero-Shot Isomer Recognition}

K-M\textsuperscript{3}AID stands out prominently when compared to SP, WP, and no communication approaches in the task of zero-shot isomer recognition, achieving an exceptional 100\% accuracy across given groups of isomers (refer to Table~\ref{tab:molecule-match}). These empirical observations underscore the advantages of K-M\textsuperscript{3}AID in the context of isomer recognition. The superiority of K-M\textsuperscript{3}AID over the no communication baseline demonstrates the positive impact of node-level alignment on graph-level alignment, emphasizing the potency of meta-learning.

\setlength{\tabcolsep}{1pt}
\begin{table}[!ht]
    \centering
    \small
    \caption{Zero-shot isomer recognition accuracy (\%) of K-M\textsuperscript{3}AID and baselines.}
    \begin{tabular}{cccccc}
    \hline
        Formula & \#Isomers & SP & WP (th=1) & No Comm. & K-M\textsuperscript{3}AID \\ \hline
        $\text{C}_{4}\text{H}_{6}\text{O}$ & 15  &86.7 & 86.7 & 86.7 & \textbf{100.0} \\
        $\text{C}_{9}\text{H}_{9}\text{N}$ & 15 &86.7 & 80.0 & \textbf{100.0} & \textbf{100.0} \\ 
        $\text{C}_{7}\text{H}_{11}\text{NO}_{3}$ & 14 &78.6 & 85.7 & 85.7 & \textbf{100.0} \\ 
        $\text{C}_{6}\text{H}_{13}\text{NO}$ & 23 &91.3 & 91.3 & \textbf{100.0} & \textbf{100.0} \\ 
        $\text{C}_{8}\text{H}_{7}\text{NO}_{4}$ & 13 &92.3 & 84.6 & 92.3 & \textbf{100.0} \\ 
        $\text{C}_{15}\text{H}_{24}\text{O}$ & 16  &93.8 & 93.8 & \textbf{100.0} & \textbf{100.0} \\ 
        $\text{C}_{11}\text{H}_{14}$ & 10 &90.0 & 80.0 & 70.0 & \textbf{100.0} \\ 
        $\text{C}_{7}\text{H}_{15}\text{NO}$ & 14&85.7 & 85.7 &  \textbf{100.0} & \textbf{100.0} \\ 
        $\text{C}_{10}\text{H}_{16}\text{O}_{2}$ & 26 &92.3 & 84.6 & \textbf{100.0} & \textbf{100.0} \\ 
        $\text{C}_{8}\text{H}_{15}\text{N}$ & 11 &81.8 & 90.9 & \textbf{100.0} & \textbf{100.0} \\ 
        %$\text{C}_{8}\text{H}_{10}$ & 16 &93.75 & 87.50 & 100.00 \\ 
        \hline
    \end{tabular}
    \label{tab:molecule-match}
    
\end{table}

\subsubsection {Performance on Zero-Shot Peak Assignment}

The K-M\textsuperscript{3}AID model demonstrates a validation accuracy surpassing 90\% for peak assignment (peak-atom alignment) within the node-level alignment module after 200 epochs (see Figure~\ref{fig:atom-alignment-acc}.A). Notably, the model achieves a 100\% accuracy rate in 74.1\% of molecules containing fewer than 10 carbon atoms (see Figure~\ref{fig:atom-alignment-acc}.B). For molecules with carbon atom counts ranging from 10 to 20, the model attains 100\% accuracy in 37.2\% of cases (see Figure~\ref{fig:atom-alignment-acc}.C). Furthermore, it achieves an accuracy exceeding 80\% in more than 50\% of cases pertaining to molecules containing more than 20 carbon atoms (see Figure~\ref{fig:atom-alignment-acc}.D). Additionally, to further illustrate the power of K-M\textsuperscript{3}AID on peak assignment, we present two complex natural product molecules featuring multiple rings (4 and 4, respectively) and stereogenic (chiral) centers (6 and 8, respectively) in Figure~\ref{fig:atom-alignment-specific}.

\begin{figure*}[ht!]
    \centering
    \includegraphics[width=0.95\textwidth]{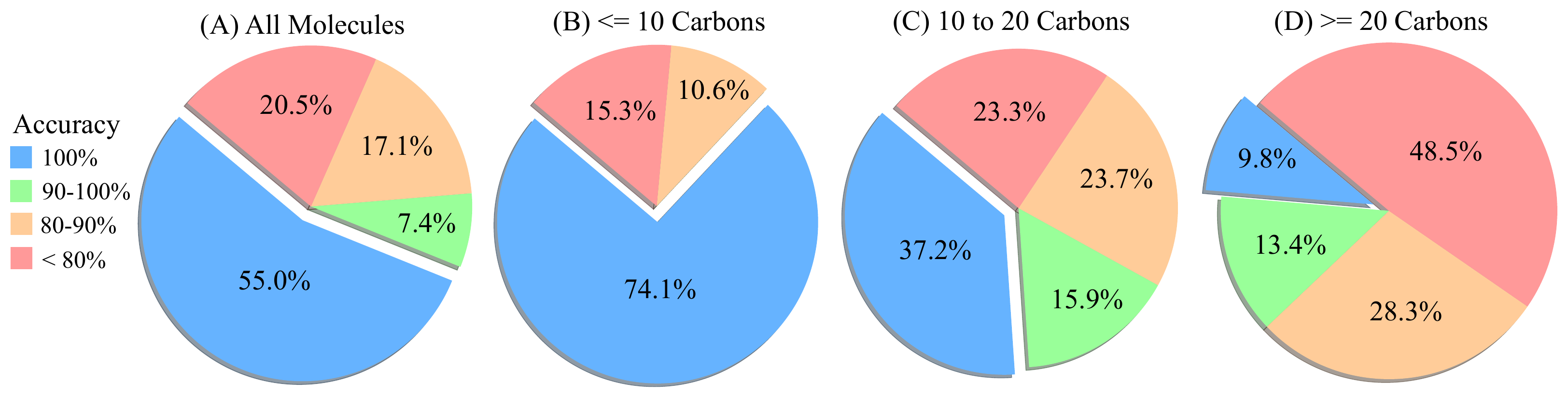}
    \caption{The statistics of zero-shot peak assignment.}
    \label{fig:atom-alignment-acc}
\end{figure*}

K-M\textsuperscript{3}AID demonstrates superior performance in peak assignment compared to SP and WP. Our case studies reveal that the limitations of SP and WP become particularly pronounced in two scenarios: 1) When local contexts of specific atoms exhibit a high degree of similarity.
2) When certain atoms display symmetric mapping within the same molecule.

In the former scenario, exemplified by molecular A in Figure~\ref{fig:symmetric-comparison}, atom 0 and atom 4 are secondary carbons (attaching to 2 carbons and 2 hydrogens), nearly symmetric on the same 5-member ring, corresponding to the peak position measured in ppm of 27.0 and 29.8, respectively (for the definition of ppm, please refer to Section \ref{Knowledge-Span-ppm}). The similar local content of these two atoms fools SP and WP. Meanwhile, atom 1 and atom 3 are tertiary carbons (attaching to 3 carbons and 1 hydrogen), nearly symmetric on the same 5-member ring, corresponding to the peak position measured in ppm of 54.5 and 44.1, respectively. Only WP fails to distinguish and align them.

In the latter scenario, exemplified by molecular B in Figure~\ref{fig:symmetric-comparison}, there exist instances one-to-one and one-to-many for atomic-level alignment within the molecular configuration. Both SP and WP methods misalign certain atoms with other atoms with small ppm differences (less than 3 ppm in this case), rather than aligning them with themselves or their symmetric counterparts. In contrast, the K-M\textsuperscript{3}AID approach excels in both scenarios by discerning each one of the atoms, which is attributed to the full utilization of ppm difference distance learning. (For additional cases, please refer to Appendix Figure~\ref{fig:molecule-compare-appendix})

% In  molecular C in Figure, similar with molecule A, atom 13 and atom 14 are tertiary carbons (attaching to 3 carbons and 1 hydrogen) and on the same 5-member ring, corresponding to the ppm of 34.3 and 35.6, respectively. The similar local content of these two atoms fools SP and WP. In addition, WP fails with more atomic alignments. Similar with molecule B, the molecular D is chemical symmetric regarding atom 0. Thus, atom 1 and atom 3 correspond to the same peak on the spectra. The ppm of atom 1 and atom 3 is 114.2, the ppm of atom 2 and atom 4 is 110.0. While there is 4.2 difference, SP and WP fails to pick up right alignment for atom 1 and atom 3. In contrast, K-ID succeed to align the atoms with peaks in both molecules. 

\begin{figure*}[ht!]
    \centering
    \includegraphics[width=0.95\textwidth]{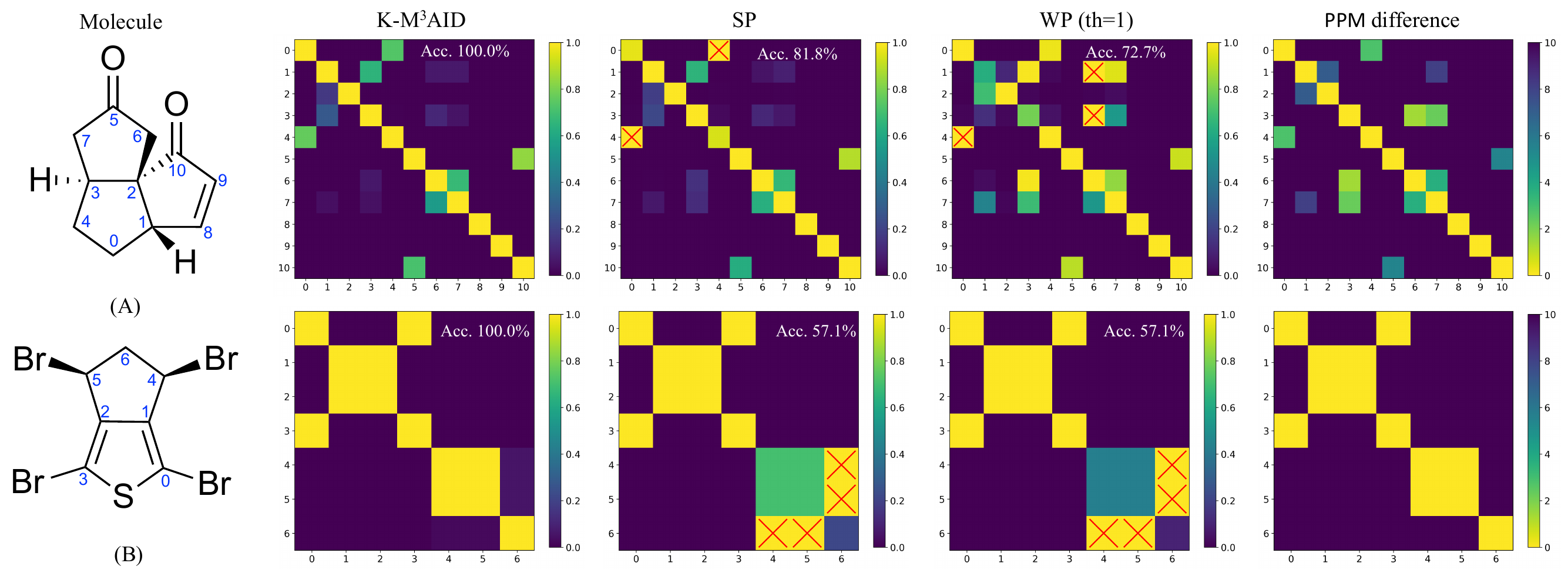}
    \caption{Case study of peak assignment. Yellow cells in PPM difference represent the ground truth alignment, and red cross represents the wrong alignment. For the definition of ppm, please refer to \ref{Knowledge-Span-ppm}. For additional cases, please refer to Appendix Figure~\ref{fig:molecule-compare-appendix}}
    \label{fig:symmetric-comparison}
\end{figure*}

\section{Related Work}
\label{sec2}

\textbf{Multimodal Instance-Wise Discrimination:} As mentioned in the preliminaries, instance discrimination \cite{le2020ContraRepSurvey,zolfaghari2021crossclr, morgado2021audio,liu2023SelfContrast}, an important part of contrastive learning, distinguishes individual instances without explicit class labels. Transitioning into multimodal contrastive learning, it can be categorized into two general approaches: strong-pair-based \cite{oord2019representation,jaiswal2021contraself,liu2023SelfContrast} and weak-pair-based \cite{salakhutdinov2007softneighbor,frosst2019analyzing,liang2021metaalignment} instance-wise discrimination. The strong-pair-based approach, such as the Noise Contrastive Estimation (NCE) method, enforces a precise one-to-one correspondence for real samples with artificially generated noise samples. An example of a positive pair can be a noise-added picture of a zebra with the text description of a zebra. Instead of one-to-one correspondences, the weak-pair-based approach relaxes the positive pairs to broader semantic correspondences. An example of a positive pair can be a picture of a zebra with the text description of a horse but not with the text description of a tiger.

% \textbf{Multimodal Meta-Alignment:} As viewed through the lenses of intermediate-level alignment and irreducible element-level alignment, multimodal meta-alignment represents a multifaceted approach to ensuring organizational coherence and effectiveness \cite{ma2022metamodal}). Exemplary instances of intermediate-level meta-alignment, as seen in works like Cross-Modal Generalization \cite{chen2017deep,li2020unimo,liang2021metaalignment,zhang2021cross}) and Livestreaming Product Recognition \cite{yang2023crossview}), typically function at both the objective level and the patch level. The exploration of multimodal meta-alignment at the level of irreducible element remains relatively underdeveloped in the current landscape.

%\textcolor{red}{Cite the original reference of Meta-alignment}
\textbf{Multimodal Meta-Alignment:} 
Within the realm of multimodal alignment, multimodal meta-alignment is a novel method for aligning representation spaces using paired cross-modal data with different similarity levels while ensuring quick generalization to new tasks across different modalities \cite{liang2021metaalignment}. This approach can be observed at different levels, including the intermediate and fundamental (irreducible) element level. Examples of this method at the intermediate level can be found in research on Cross-Modal Generalization \cite{chen2017deep,li2020unimo,liang2021metaalignment,zhang2021cross} and Livestreaming Product Recognition \cite{yang2023crossview}. While these studies showcase how multimodal meta-alignment operates at a broad objective level, the application of multimodal meta-alignment at the most fundamental element level remains underexplored in current research.

\section{Conclusion and Future Work}
\label{sec5}
In this paper, we introduced the K-M\textsuperscript{3}AID (Knowledge-Guided Multi-Level Multimodal Alignment with Instance-Wise Discrimination) framework, incorporating both graph-level and node-level alignment. Its effectiveness was demonstrated through multiple zero-shot tasks, including molecular retrieval, isomer recognition, and peak assignment. The significance of knowledge-guided instance-wise discrimination is underscored through various metrics and case studies. Moreover, the findings from molecular retrieval and isomer recognition highlight the favorable influence of node-level alignment on graph-level alignment. This emphasizes the successful integration of meta-learning within our hierarchical alignment framework. While our framework achieves an atomic-level alignment overall accuracy of 100\% for 55\% of cases, it drops significantly to 9.8\% when handling molecules with more than 20 carbon atoms. Currently, our graph encoder operates on 2D molecular graphs with basic node and edge features. This implementation potentially constrains its ability to generate precise node embeddings for distinguishing atoms in highly complex scenarios. Future developments could benefit from incorporating a 3D-based graph, holding substantial potential to enhance performance in such complex situations.

\section*{Accessibility}
The code and dataset will be made available upon the date of publication.

{\small
\bibliography{CLMA}}
\bibliographystyle{icml2024}

\newpage
\appendix
\onecolumn
\begin{center}
    \textbf{\Large{Appendix}}
\end{center}
\counterwithin{figure}{section}
\counterwithin{equation}{section}
\counterwithin{table}{section}
\setcounter{figure}{0} % Restart figure numbering
\setcounter{equation}{0} % Restart figure numbering
\setcounter{table}{0} % Restart figure numbering

\section{Revisiting Knowledge Span Guided Loss}
\label{appendix:knowledge-span-guide-proof}
\begin{theorem}[Knowledge Span Guided Loss]
Suppose $\mathcal{M}$ is the set of instances. $\mathcal{A} \subset \mathbb{R}^{d_1}$ is the set of tunable instances' embeddings in modality A, $\mathcal{B} \subset \mathbb{R}^{d_1}$ is the set of tuable instances' embeddings in modality B, and $\mathcal{K} \subset \mathbb{R}^{d_2}$ is the corresponding fixed knowledge span label that can guide the relative distance learning between components in $\mathcal{A}$ and $\mathcal{B}$. Thus, the size of $\mathcal{A}$, $\mathcal{B}$, $\mathcal{K}$ are $|\mathcal{M}|$, respectively. 
%In particular, the embeddings in $\mathcal{P}$ is given and fixed, while those in $\mathcal{A}$ and $\mathcal{B}$ are tunables. 

Let $\mathcal{A}_i$ be the $i^{th}$ instance embedding of $\mathcal{A}$, and $\mathcal{B}_j$ be the $j^{th}$ instance embedding of $\mathcal{B}$. We define the distance function between $\mathcal{A}_i$ and $\mathcal{B}_j$ as $d_{E}(\mathcal{A}_i, \mathcal{B}_j) = \mathcal{A}_i \cdot \mathcal{B}_j \rightarrow \mathbb{R}^{+}$, and calibration function $d(\mathcal{K}_{i}, \mathcal{K}_{j}) \rightarrow \mathbb{R}^{+}$ with a monotonic property and constraint $\sum_{j=1}^{|\mathcal{M}|}d(\mathcal{K}_{i}, \mathcal{K}_{j}) = 1$, in which $\mathcal{K}_{i}$ and $\mathcal{K}_{j}$ serve as the designated Knowledge Span Label. We introduce the Knowledge Span Guided Loss (KSGL) as follows:
% \begin{equation}
%     CL_{IE}(i) = KSGL(i) =\displaystyle\sum_{\substack{1\leq j \leq |\mathcal{M}|}} - d(K(\mathbf{A}_i), K(\mathbf{B}_j)) \log \frac{e^{d_{E}(A_{i}, B_{j})}}{\displaystyle\sum_{\substack{1\leq k \leq |\mathcal{M}|}} e^{d_{E}(A_{i}, B_{k})}}
%     \label{equ:ksgl}
% \end{equation}
\begin{align}
     KSGL(i) &=-\displaystyle\sum_{\substack{1\leq j \leq |\mathcal{M}|}} d(\mathcal{K}_{i}, \mathcal{K}_{j}) \log \frac{e^{d_{E}(\mathcal{A}_{i}, \mathcal{B}_{j})}}{\displaystyle\sum_{\substack{1\leq k \leq |\mathcal{M}|}} e^{d_{E}(\mathcal{A}_{i}, \mathcal{B}_{k})}} \\
    & = -\displaystyle\sum_{\substack{1\leq j \leq |\mathcal{M}|}} d(\mathcal{K}_{i}, \mathcal{K}_{j}) \log (\text{softmax}(d_{E}(\mathcal{A}_{i}, \mathcal{B}_{j})))
    \label{equ:ksgl-revisit}
\end{align}
\end{theorem}
\begin{proof}
In order to optimize the loss $KSGL(i)$, we need to set the following partial derivative to be 0 for each $d_{E}(\mathcal{A}_{i}, \mathcal{B}_{j})$ with $1\leq j \leq |\mathcal{M}|$. Here are the detail process: 
\begin{align*}
\frac{\partial KSGL(i)}{\partial {d_{E}(\mathcal{A}_{i}, \mathcal{B}_{j})}} 
&= \frac{\partial}{\partial d_{E}(\mathcal{A}_{i}, \mathcal{B}_{j})}\underbrace{\left( - d(\mathcal{K}_{i}, \mathcal{K}_{j}) \log \frac{e^{d_{E}(\mathcal{A}_{i}, \mathcal{B}_{j})}}{e^{d_{E}(\mathcal{A}_{i}, \mathcal{B}_{j})} + \sum_{\substack{k \neq j}} e^{d_{E}(\mathcal{A}_{i}, \mathcal{B}_{k})}} \right)}_{\text{When the numerator includes } e^{d_{E}(\mathcal{A}_{i}, \mathcal{B}_{j})}} \\
&\quad + \frac{\partial}{\partial d_{E}(\mathcal{A}_{i}, \mathcal{B}_{j})}\underbrace{\left( \sum_{\substack{k \neq j}} - d(\mathcal{K}_{i}, \mathcal{K}_{k}) \log \frac{e^{d_{E}(\mathcal{A}_{i}, \mathcal{B}_{k})}}{e^{d_{E}(\mathcal{A}_{i}, \mathcal{B}_{j})} + \sum_{\substack{k \neq j}} e^{d_{E}(\mathcal{A}_{i}, \mathcal{B}_{k})}} \right)}_{\text{When the numerator does not include }e^{d_{E}(\mathcal{A}_{i}, \mathcal{B}_{j})}} \\
&= -(d(\mathcal{K}_{i}, \mathcal{K}_{j}) - d(\mathcal{K}_{i}, \mathcal{K}_{j}) \cdot \text{softmax}(d_{E}(\mathcal{A}_{i}, \mathcal{B}_{j})) \\
&\quad - \sum_{\substack{k \neq j}} d(\mathcal{K}_{i}, \mathcal{K}_{k}) \cdot \text{softmax}(d_{E}(\mathcal{A}_{i}, \mathcal{B}_{j})) \\
&= - \left( d(\mathcal{K}_{i}, \mathcal{K}_{j}) - (d(\mathcal{K}_{i}, \mathcal{K}_{j}) + \sum_{\substack{k \neq j}} d(\mathcal{K}_{i}, \mathcal{K}_{k})) \cdot \text{softmax}(d_{E}(\mathcal{A}_{i}, \mathcal{B}_{j})) \right)
\end{align*}
Since $\sum_{l=1}^{|\mathcal{M}|}d(\mathcal{K}_i, \mathcal{K}_l)$ = 1, we can further simplify it as 
\begin{align*}
\frac{\partial KSGL(i)}{\partial {d_{E}(\mathcal{A}_{i}, \mathcal{B}_{j})}} = - (d(\mathcal{K}_{i}, \mathcal{K}_{j}) -  \text{softmax}(d_{E}(\mathcal{A}_{i}, \mathcal{B}_{j}))
\end{align*}
In order to optimize, we need to set the respective partial derivative to be 0:
\begin{align*}
\frac{\partial KSGL(i)}{\partial {d_{E}(\mathcal{A}_{i}, \mathcal{B}_{j})}} = - (d(\mathcal{K}_{i}, \mathcal{K}_{j}) - \text{softmax}(d_{E}(\mathcal{A}_{i}, \mathcal{B}_{j})) = 0
\end{align*}
In addition, the corresponding second partial derivative denoted as $\frac{\partial KSGL(i)}{\partial {d_{E}^2(\mathcal{A}_{i}, \mathcal{B}_{j})}}$ manifests as follows:
\begin{align*}
\frac{\partial KSGL(i)}{\partial {d_{E}^2(\mathcal{A}_{i}, \mathcal{B}_{j})}} = \text{softmax}(d_{E}(\mathcal{A}_{i}, \mathcal{B}_{j}))(1- \text{softmax}(d_{E}(\mathcal{A}_{i}, \mathcal{B}_{j})))
\end{align*}
As $\text{softmax}(d_{E}(\mathcal{A}_{i}, \mathcal{B}_{j}))$ takes values within the open interval (0,1), it follows that $\frac{\partial KSGL(i)}{\partial {d_{E}^2(\mathcal{A}_{i}, \mathcal{B}_{j})}}$ is always positive. Consequently, the pinnacle of optimization emerges as a global minimum.\\
Furthermore, when it comes to optimum:
\begin{align*}
d(\mathcal{K}_{i}, \mathcal{K}_{j}) &= \text{softmax}(d_{E}(\mathcal{A}_{i}, \mathcal{B}_{j})) \\
d_{E}(\mathcal{A}_{i}, \mathcal{B}_{j}) &= \log(d(\mathcal{K}_{i}, \mathcal{K}_{j})) + \log \left( \sum_{\substack{1\leq l \leq |\mathcal{M}|}} e^{d_{E}(\mathcal{A}_{i}, \mathcal{B}_{l})} \right)
\end{align*}
It is easy to show that when it reaches optimum, $d_{E}(A_{i}, B_{j})$ is consistent with Knowledge Span Guidance $d(\mathcal{K}_{i}, \mathcal{K}_{j})$. Without loss of generosity, suppose $d(\mathcal{K}_{i}, \mathcal{K}_{j}) > d(\mathcal{K}_{i}, \mathcal{K}_{j'})$ :
\begin{align*}
d_{E}(\mathcal{A}_{i}, \mathcal{B}_{j}) - d_{E}(\mathcal{A}_{i}, \mathcal{B}_{j'}) &= \log(d(\mathcal{K}_{i}, \mathcal{K}_{j})) + \log \left( \sum_{\substack{1\leq l \leq |\mathcal{M}|}} e^{d_{E}(\mathcal{A}_{i}, \mathcal{B}_{l})} \right) \\
&\quad - \left( \log(d(\mathcal{K}_{i}, \mathcal{K}_{j'})) + \log \left( \sum_{\substack{1\leq l \leq |\mathcal{M}|}} e^{d_{E}(\mathcal{A}_{i}, \mathcal{B}_{l})} \right) \right) \\
&= \log(d(\mathcal{K}_{i}, \mathcal{K}_{j})) - \log(d(\mathcal{K}_{i}, \mathcal{K}_{j'})) \\
&= \log\left( \frac{d(\mathcal{K}_{i}, \mathcal{K}_{j})}{d(\mathcal{K}_{i}, \mathcal{K}_{j'})} \right) > 0
\end{align*}
\end{proof}

% graph/node	Spectrum	Peak
% AtomicNum, ChiralTag,hybridization	ppm	ppm
% BondType, BondDirection		multiplicity

% \begin{table}[!ht]
% \caption{GIN structure and projection ablation study}
% \label{tab:gin-structure-projection-ablation}
% \small
% \centering
%     \begin{tabular}{cccc}
%     \hline
%     \bf GIN Depth  & \bf GIN Embedding Dim  & \bf Projection Dim & \bf Validation accuracy (\%) \\ 
%     \hline
%     3 & 128 & 128 & 86.6 \\

%     \hline
%     \end{tabular}
% \end{table}

\section{Experimental Setting}
\label{appendix:exp-setting}

\subsection{Pre-training}
\textbf{Dataset:}
 We use $^{13}$C NMR spectra of about 20,000 molecules sourced from nmrshiftdb2 \cite{steinbeck2003nmrshiftdb}), a public access database that contains NMR spectra of organic molecules. In the collected dataset, molecule are aligned with their respective $^{13}$C NMR spectra, and atomic alignments with peaks are also included. Notably, the dataset contains 12,771 molecules with fewer than 10 carbon atoms, 7,043 molecules featuring carbon atom counts ranging from 10 to 20, and 1,138 molecules incorporating more than 20 carbon atoms. The quality of the dataset was further validated by experienced organic chemists. We randomly sample 80\% of the molecules for training and the rest for evaluation.

\textbf{Training:}
  We concurrently leverage both graph- and node-level alignment tasks in the pre-training of K-M\textsuperscript{3}AID.  Graph-level alignment focuses on aligning molecules with their spectra, accompanied by the utilization of cross-entropy loss for contrastive purposes. On the other hand, node-level alignment entails aligning atoms with their corresponding peaks, implemented through knowledge-guided instance-wise discrimination to achieve contrastive loss. A diverse set of molecular features is employed for training, including atomic number (node feature), chiral tags (node feature), hybridization (node feature), bond types (edge feature), and bond direction (edge feature). The spectral features are derived from  peak intensity, peak position (chemical shift measured in ppm), and peak type.

\subsection{Zero-Shot  Molecular Retrieval}

\textbf{Dataset:}
 We randomly collected 1 million molecules from PubChem \cite{kim2023pubchem} to form a molecular reference library. Subsequently, we carefully selected 1000 spectra, ensuring that they had not appeared in the training dataset, from an external dataset to serve as query spectra. Following this, the corresponding molecules associated with these 1000 spectra were added into the existing reference library.

\textbf{Evaluation:}
 We perform molecular retrieval using each of the selected spectra to determine if the correct corresponding molecular entity can be retrieved from the reference library. The model's performance is assessed at top-1 accuracy, as well as at accuracy of top 5, top 10, and top 25.

\subsection{Zero-Shot Isomer Recognition}
\textbf{Dataset:}
We categorize isomers from the validation dataset to guarantee their absence from the training dataset. To assess effectiveness, we perform isomer recognition on each isomer group containing at least 10 molecules. Within the same group, isomers may be structural or spatial isomers of each other. Structural isomers refer to molecules with the same molecular formula but different structural arrangements of atoms, resulting in distinct chemical structures. On the other hand, spatial isomers, also known as stereoisomers, have the same molecular formula and arrangement of atoms but differ in the spatial orientation of their atoms in three-dimensional space, leading to different stereoisomeric forms. (For an elucidation of an isomer group and in-depth insights into isomers with NMR, please refer to the details provided in Appendix~\ref{appendix:more-isomer-discussion}.)

\textbf{Evaluation:}
 We conduct isomer recognition for each isomer group, aiming to assess the correct alignment of each spectrum with its respective molecule within each isomer group.

\subsection{Zero-Shot Peak Assignment}
\textbf{Dataset:} 
 The dataset utilized for evaluating the overall performance of K-M\textsuperscript{3}AID on zero-shot peak assignment is the validation dataset from the pre-training phase. In order to highlight the capabilities of K-M\textsuperscript{3}AID in zero-shot peak assignment, the case studies include complex natural products featuring multiple fused rings, stereogenic (chiral) centers, and symmetric structures.

\textbf{Evaluation:}
 We conduct peak assignment within each molecule, aiming to assess the accurate alignment of each atom with its corresponding peak on the spectrum. It's important to note that this alignment process is confined to each individual molecule and not across different molecules.

\section{Further ablation study about parameter choices}

\subsection{Ablation study about the choice of GIN structure and projection.}
\label{appendix:gin-projection-parameter}
 We choose GIN\cite{xu2018powerful} as our graph encoder. By Table~\ref{tab:gin-structure-projection-ablation}, "GIN Depth" signifies the number of layers in the GIN, "GIN Embedding Dim" denotes the dimensionality of the embeddings generated by the GIN model, and "Projection Dim" indicates the resulting dimensionality after transforming the GIN-produced embeddings. In particular, the best performance is observed when the GIN model has 5 layers, GIN Embedding Dim is 128, and projection Dim is 512. 
\begin{table}[!ht]
\caption{GIN structure and projection ablation study}
\label{tab:gin-structure-projection-ablation}
\small
\centering
    \begin{tabular}{cccc}
    \hline
    \bf GIN Depth  & \bf GIN Embedding Dim  & \bf Projection Dim & \bf Validation accuracy (\%) \\ 
    \hline
    3 & 128 & 128 & 86.6 \\
    3 & 256 & 128 & 86.8 \\
    3 & 512 & 128 & 86.3 \\
    5 & 128 & 128 & 89.4 \\
    5 & 256 & 128 & 89.6 \\
    5 & 512 & 128 & 89.3 \\
    3 & 128 & 256 & 86.6 \\
    3 & 256 & 256 & 86.8 \\
    3 & 512 & 256 & 86.3 \\
    5 & 128 & 256 & 89.4 \\
    5 & 256 & 256 & 89.6 \\
    5 & 512 & 256 & 89.3 \\
    3 & 128 & 512 & 86.6 \\
    3 & 256 & 512 & 86.5 \\
    3 & 512 & 512 & 86.2 \\
    5 & 32 & 512 & 84.0 \\
    5 & 64 & 512 & 87.5 \\
    5 & 128 & 512 & \textbf{90.0} \\
    5 & 256 & 512 & 89.4 \\
    5 & 512 & 512 & 88.9 \\
    \hline
    \end{tabular}
\end{table}

% \subsection{Ablation study about the choice of $\tau_1$ and $\tau_2$.}
% \subsection{Ablation Study on the Choice of $\tau_1$ and $\tau_2$}
\subsection{Ablation Study on the Choice of \texorpdfstring{$\tau_1$ and $\tau_2$}{tau1 and tau2}}

\label{appendix:tau-ablation-study}
We also conducted a further ablation study exploring different combinations of $\tau_1$ and $\tau_2$ as shown in Table~\ref{tab:temp-sensitivity}. For this analysis, we fixed the GIN depth at 5, set the GIN embedding dimensionality to 128, and maintained a projection dimension of 512.  We observe that the best performance is achieved when $\tau_1$ = $10^{-5}$ and $\tau_2$ = $10^1$.

\begin{table}[ht]
\caption{Ablation study about $tau_1$ and $tau_2$. We have 5 layers and 128 dimension as the final representation.} 
\label{tab:temp-sensitivity}
\centering
\small
\begin{tabular}{cccc}
\hline
\bf $\tau_1$  & \bf $\tau_2$  & \bf Molecular Alignment Accuracy (\%) & \bf Atom Alignment Accuracy (\%) \\ 
\hline
$10^{-1}$ & $10^{1}$ &  94.9 & 89.6 \\
$10^{-1}$ & $10^{2}$ &  95.2 & 89.8 \\
$10^{-1}$ & $10^{3}$ &  95.6 & 89.6\\
$10^{-1}$ & $10^{4}$ &  95.1 & 88.9\\
$10^{-1}$ & $10^{5}$ &  95.0 & 89.3\\
$10^{-2}$ & $10^{1}$ &  95.5 &89.8 \\
$10^{-2}$ & $10^{2}$ &  94.8 & 89.8\\
$10^{-2}$ & $10^{3}$ &  95.4 & 88.8\\
$10^{-2}$ & $10^{4}$ &  94.8 &87.2\\
$10^{-2}$ & $10^{5}$ &  95.1 &89.4\\
$10^{-3}$ & $10^{1}$ &  95.0 &89.2\\
$10^{-3}$ & $10^{2}$ &  95.1 &89.1\\
$10^{-3}$ & $10^{3}$ &  95.2 &89.0\\
$10^{-3}$ & $10^{4}$ &  95.3 &89.7\\
$10^{-3}$ & $10^{5}$ &  95.0 &89.4\\
$10^{-4}$ & $10^{1}$ &  95.0 &89.8\\
$10^{-4}$ & $10^{2}$ &  95.1 &89.7\\
$10^{-4}$ & $10^{3}$ &  95.0 &89.8\\
$10^{-4}$ & $10^{4}$ &  95.3 &89.5\\
$10^{-4}$ & $10^{5}$ &  95.1 &88.4\\
$10^{-5}$ & $10^{1}$ &  \textbf{95.4} & \textbf{90.0}\\
$10^{-5}$ & $10^{2}$ &  95.0 &89.5\\
$10^{-5}$ & $10^{3}$ &  95.8 & 89.6\\
$10^{-5}$ & $10^{4}$ &  95.2 & 89.7\\
$10^{-5}$ & $10^{5}$ &  95.0 & 89.7\\
\hline
\end{tabular}
\end{table}

\section{Additional Results on Molecular Retrieval}

\begin{table}[!ht]
    \centering
    \small
    \caption{Zero-shot molecular retrieval top 5, 10, 25 accuracy (\%) with K-M\textsuperscript{3}AID and baselines}
        \begin{tabular}{c|c|c|c|c|c|c}
        \hline
         
        \multicolumn{1}{c|}{Method}     &  \multicolumn{1}{l|}{Accuracy (\%)} %& \multicolumn{1}{c|}{10} 
        & \multicolumn{1}{c|}{$10^2$} & \multicolumn{1}{c|}{$10^3$} & \multicolumn{1}{c|}{$10^4$} & \multicolumn{1}{c|}{$10^5$} & \multicolumn{1}{c}{$10^6$} \\ \hline
        \multirow{4}{*}{K-M\textsuperscript{3}AID} & Top 1 & 95.8$\pm$1.0 & 80.4$\pm$3.9 & 46.3$\pm$1.2 & 18.0$\pm$0.8 & 5.8$\pm$1.7 \\
        &\multirow{1}{*}{Top 5} & 99.8$\pm$0.2 & 96.8$\pm$0.5 & 77.8$\pm$1.1 & 41.6$\pm$1.6 & 16.6$\pm$2.3 \\
        &\multirow{1}{*}{Top 10} & 100.0$\pm$0.0 & 98.8$\pm$0.3 & 87.7$\pm$1.0 & 53.9$\pm$1.8 & 25.2$\pm$3.2 \\
        &\multirow{1}{*}{Top 25} & 100.0$\pm$0.0 & 99.6$\pm$0.2 & 94.8$\pm$0.6 & 71.6$\pm$0.6 & 37.8$\pm$4.1 \\
                                         \hline
        \multirow{4}{*}{SP}     & Top 1                 %& 99.7$\pm$0.22                    
        & 95.3$\pm$0.8                    & 78.6$\pm$2.7                      & 35.8$\pm$3.8                       & 12.9$\pm$1.6                        & 3.4$\pm$0.9                          \\
                                         & Top 5                %& 100.0$\pm$0.00                   
                                         & 95.4$\pm$0.1                     & 77.3$\pm$0.7                     & 44.7$\pm$2.3                       & 16.2$\pm$2.4                        & 4.4$\pm$1.5                          \\
                                         & Top 10                %& 100.0$\pm$0.00                   
                                         & 100.0$\pm$0.0                   & 97.3$\pm$0.7                      & 77.5$\pm$2.3                       & 40.2$\pm$2.4                        & 12.3$\pm$1.5                         \\
                                         & Top 25                %& 100.0$\pm$0.00                   
                                         & 100.0$\pm$0.0                    & 99.1$\pm$0.2                      & 85.9$\pm$1.0                       & 53.1$\pm$3.0                        & 18.5$\pm$1.8                         \\
                                         \hline
                \multirow{4}{*}{WP(th=1)} & Top 1                 %& 99.0$\pm$0.29                    
        & 92.9$\pm$0.6                     & 71.7$\pm$1.0                      & 32.7$\pm$1.3                       & 10.7$\pm$0.5                        & 3.6$\pm$0.7                          \\
                                         & Top 5                 %& 100.0$\pm$0.00                   
                                         & 99.6$\pm$0.1                    & 93.8$\pm$0.8                      & 63.9$\pm$1.5                       & 29.3$\pm$1.5                        & 10.2$\pm$1.2                         \\
                                         & Top 10                %& 100.0$\pm$0.00                   
                                         & 99.9$\pm$0.0                    & 97.1$\pm$0.4                      & 76.8$\pm$0.7                       & 39.3$\pm$0.9                        & 15.7$\pm$1.5                         \\
                                         & Top 25                %& 100.0$\pm$0.00                  
                                         & 100.0$\pm$0.0                    & 99.1$\pm$0.2                      & 88.2$\pm$0.6                       & 55.7$\pm$1.1                        & 26.5$\pm$2.0\\
                                         \hline
                 \multirow{4}{*}{No comm} & Top 1                 %& 99.0$\pm$0.29                    
        & 94.8$\pm$1.2                     & 77.6$\pm$1.4                      & 40.1$\pm$1.2                       & 14.4$\pm$0.9                        & 4.1$\pm$1.1                          \\
                                         & Top 5                %& 100.0$\pm$0.00                   
                                         & 99.8$\pm$0.1                    & 96.2$\pm$0.5                      & 73.6$\pm$2.2                       & 35.8$\pm$1.3                        & 11.4$\pm$1.0                         \\
                                         & Top 10                %& 100.0$\pm$0.00                   
                                         & 99.9$\pm$0.1                    & 98.6$\pm$0.3                      & 84.1$\pm$1.4                       & 47.3$\pm$1.8                        & 17.3$\pm$1.8                         \\
                                         & Top 25                %& 100.0$\pm$0.00                  
                                         & 100.0$\pm$0.0                    & 99.7$\pm$0.2                      & 92.9$\pm$0.9                       & 65.1$\pm$2.2                        & 27.9$\pm$2.5\\
                                         \hline
        \end{tabular}
        \label{tab:pubmed-retrievel-th-whole}
\end{table}

\begin{table}[!ht]
    \centering
    \small
    \caption{Comparing sampled molecules to their Top 1 neighbors using K-M\textsuperscript{3}AID, SP, WP (th = 1), and K-M\textsuperscript{3}AID without communication across datasets of varying sizes and employing different similarity metrics (\%) including Cosine \cite{cosine_similarity}), Dice \cite{dice_coefficient}), Russel \cite{russel_similarity}, Sokal \cite{sokal_similarity} and Tanimoto \cite{tanimoto_similarity}.}
        \begin{tabular}{c|c|c|c|c|c|c}
        \hline
         
        \multicolumn{1}{c|}{Methods}     &  \multicolumn{1}{l|}{Similarity Metric} %& \multicolumn{1}{c|}{10} 
        & \multicolumn{1}{c|}{$10^2$} & \multicolumn{1}{c|}{$10^3$} & \multicolumn{1}{c|}{$10^4$} & \multicolumn{1}{c|}{$10^5$} & \multicolumn{1}{c}{$10^6$} \\ \hline
        \multirow{4}{*}{K-M\textsuperscript{3}AID} & 
        \multirow{1}{*}{Cosine} %& Top 1(\%) 
         & 96.1$\pm$0.9 & 81.9$\pm$3.5 & 50.0$\pm$1.2 & 23.9$\pm$0.7 & 13.5$\pm$1.5 \\
                                        % \hline
        & \multirow{1}{*}{Dice}                   %& 99.7$\pm$0.22                    
        & 96.1$\pm$0.9                    & 81.9$\pm$3.5                      & 50.0$\pm$1.2                       & 23.7$\pm$0.6                        & 13.1$\pm$1.5                          \\
                                         %\hline
                &\multirow{1}{*}{Russel} &  95.8$\pm$0.1                     & 80.4$\pm$3.9                      & 46.4$\pm$1.1                       & 18.1$\pm$0.8                        & 5.9$\pm$1.7                          \\
                                         %\hline
                 &\multirow{1}{*}{Sokal} 
        & 95.9$\pm$0.9                     & 80.8$\pm$3.8                      & 47.2$\pm$1.2                       & 19.5$\pm$0.7                        & 7.7$\pm$1.7                          \\
                &\multirow{1}{*}{Tanimoto} 
        & 96.0$\pm$0.9                     & 81.2$\pm$3.7                      & 48.2$\pm$1.2                       & 21.0$\pm$0.7                        & 9.6$\pm$1.6                          \\
                                         \hline
        \multirow{4}{*}{SP}     & \multirow{1}{*}{Cosine} %& Top 1(\%) 
         & 95.4$\pm$0.8 & 78.6$\pm$2.4 & 42.5$\pm$3.6 & 20.9$\pm$1.5 & 12.2$\pm$0.7 \\
                                        % \hline
        & \multirow{1}{*}{Dice}                   %& 99.7$\pm$0.22                    
        & 95.5$\pm$0.8                    & 78.6$\pm$2.5                      & 42.4$\pm$3.6                       & 20.3$\pm$1.5                        & 11.8$\pm$0.8                          \\
                                         %\hline
                &\multirow{1}{*}{Russel} &  95.1$\pm$0.8                    & 76.9$\pm$2.7                      & 38.3$\pm$3.8                       & 14.4$\pm$1.6                        & 4.3$\pm$0.9                          \\
                                         %\hline
                 &\multirow{1}{*}{Sokal} 
        & 95.2$\pm$0.8                    & 77.4$\pm$2.6                     & 39.4$\pm$3.8                       & 15.9$\pm$1.6                        & 6.3$\pm$0.8                          \\
                &\multirow{1}{*}{Tanimoto} 
        & 95.3$\pm$0.8                    & 77.8$\pm$2.6                      & 40.4$\pm$3.7                       & 17.4$\pm$1.6                        & 8.2$\pm$0.8                          \\
                                         \hline
                \multirow{4}{*}{WP(th=1)} & \multirow{1}{*}{Cosine} %& Top 1(\%) 
         & 93.4$\pm$0.6 & 73.9$\pm$0.1  & 37.5$\pm$1.2 & 17.2$\pm$0.4 & 11.4$\pm$0.5\\
                                        % \hline
        & \multirow{1}{*}{Dice}                   %& 99.7$\pm$0.22                    
        & 93.3$\pm$0.6                    & 73.7$\pm$0.9                      & 37.3$\pm$1.2                       & 16.9$\pm$0.4                        & 11.1$\pm$0.5                          \\
                                         %\hline
                &\multirow{1}{*}{Russel} &  92.9$\pm$0.6                     & 71.7$\pm$1.0                      & 32.7$\pm$1.3                       & 10.8$\pm$0.5                        & 3.7$\pm$0.6                          \\
                                         %\hline
                 &\multirow{1}{*}{Sokal} 
        & 93.0$\pm$0.6                     & 72.3$\pm$3.4                      & 33.9$\pm$1.2                       & 12.4$\pm$0.5                        & 5.6$\pm$0.6                          \\
                &\multirow{1}{*}{Tanimoto} 
        & 93.1$\pm$0.6                     & 72.8$\pm$1.0                      & 35.1$\pm$1.2                       & 14.0$\pm$0.5                        & 7.5$\pm$0.6                          \\
                                         \hline
                 \multirow{4}{*}{No communication} & \multirow{1}{*}{Cosine} %& Top 1(\%) 
         & 95.2$\pm$1.1 & 79.2$\pm$1.3 & 44.2$\pm$1.2 & 20.6$\pm$0.7 & 12.2$\pm$0.9 \\
                                        % \hline
        & \multirow{1}{*}{Dice}                   %& 99.7$\pm$0.22                    
        & 95.1$\pm$1.1                    & 79.1$\pm$1.3                      & 44.0$\pm$1.2                       & 20.4$\pm$0.7                        & 11.8$\pm$0.9                          \\
                                         %\hline
                &\multirow{1}{*}{Russel} &  94.8$\pm$1.2                     & 77.6$\pm$1.4                      & 40.2$\pm$1.2                       & 14.5$\pm$0.9                        & 4.2$\pm$1.1                          \\
                                         %\hline
                 &\multirow{1}{*}{Sokal} 
        & 94.9$\pm$1.1                     & 78.0$\pm$1.4                      & 41.2$\pm$1.2                       & 16.0$\pm$0.8                        & 6.2$\pm$1.1                          \\
                &\multirow{1}{*}{Tanimoto} 
        & 95.0$\pm$1.1                     & 78.4$\pm$1.3                      & 42.2$\pm$1.2                       & 17.6$\pm$0.8                        & 8.2$\pm$1.0                          \\
                                         \hline
        \end{tabular}
        \label{tab:pubmed-retrievel-th-whole-similarity}
\end{table}

\section{Additional Discussion about Isomers}
\label{appendix:more-isomer-discussion}
\subsection{Isomer Category}
\label{appendix:isomer-category}
Isomers typically fall into two main categories: constitutional (structural) isomers, which share the same chemical formula but display distinct atom connectivity, and stereoisomers (spatial isomers), which share the same topology graph but diverge in their three-dimensional arrangement (see Figure~\ref{fig:isomer-division}). Constitutional isomers are NMR-variant, meaning that different isomers produce distinct NMR spectrum. In the sub-categories of stereoisomers, enantiomers are NMR-invariant, but diastereomers and cis-trans isomers are NMR-variant.

\begin{figure}[!ht]
    \centering
    \includegraphics[width=0.6\textwidth]{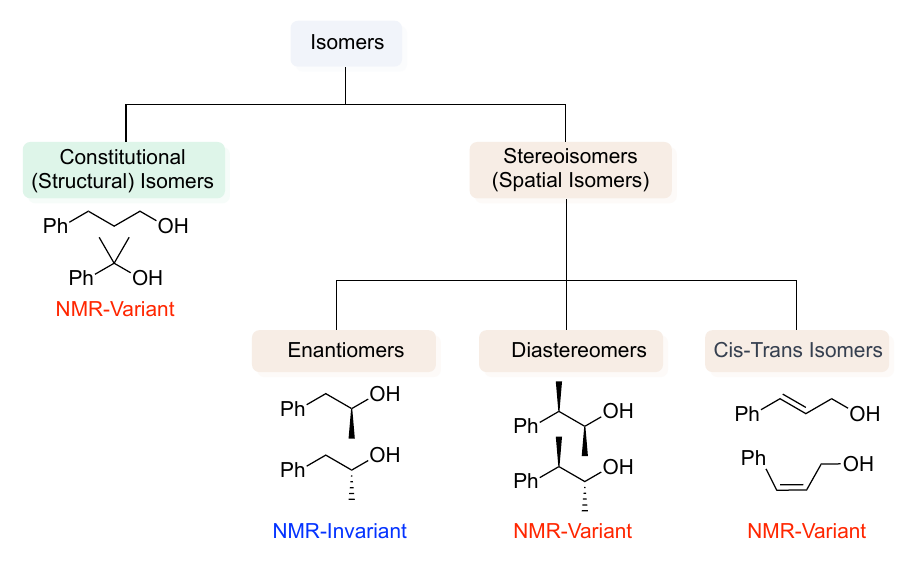}
    \caption{NMR Variability in Isomers}
    \label{fig:isomer-division}
\end{figure}

\subsection{Isomers Group for \texorpdfstring{$\text{C}_{7}\text{H}_{11}\text{NO}_{3}$}{C7H11NO3}}
\label{appendix:isomer-example}
Here is an example for isomer groups. In this isomer group of $\text{C}_{7}\text{H}_{11}\text{NO}_{3}$, they all share the same chemical formula in Figure~\ref{fig:isomer-demo}. The first 10 are constitutional (structural) isomers of each other (cycled green), the last 4 are two pairs of diastereomers (cycled brown). Each of these isomers corresponds to a distinct NMR spectrum.

\begin{figure}[!ht]
    \centering
    \includegraphics[width=0.6\textwidth]{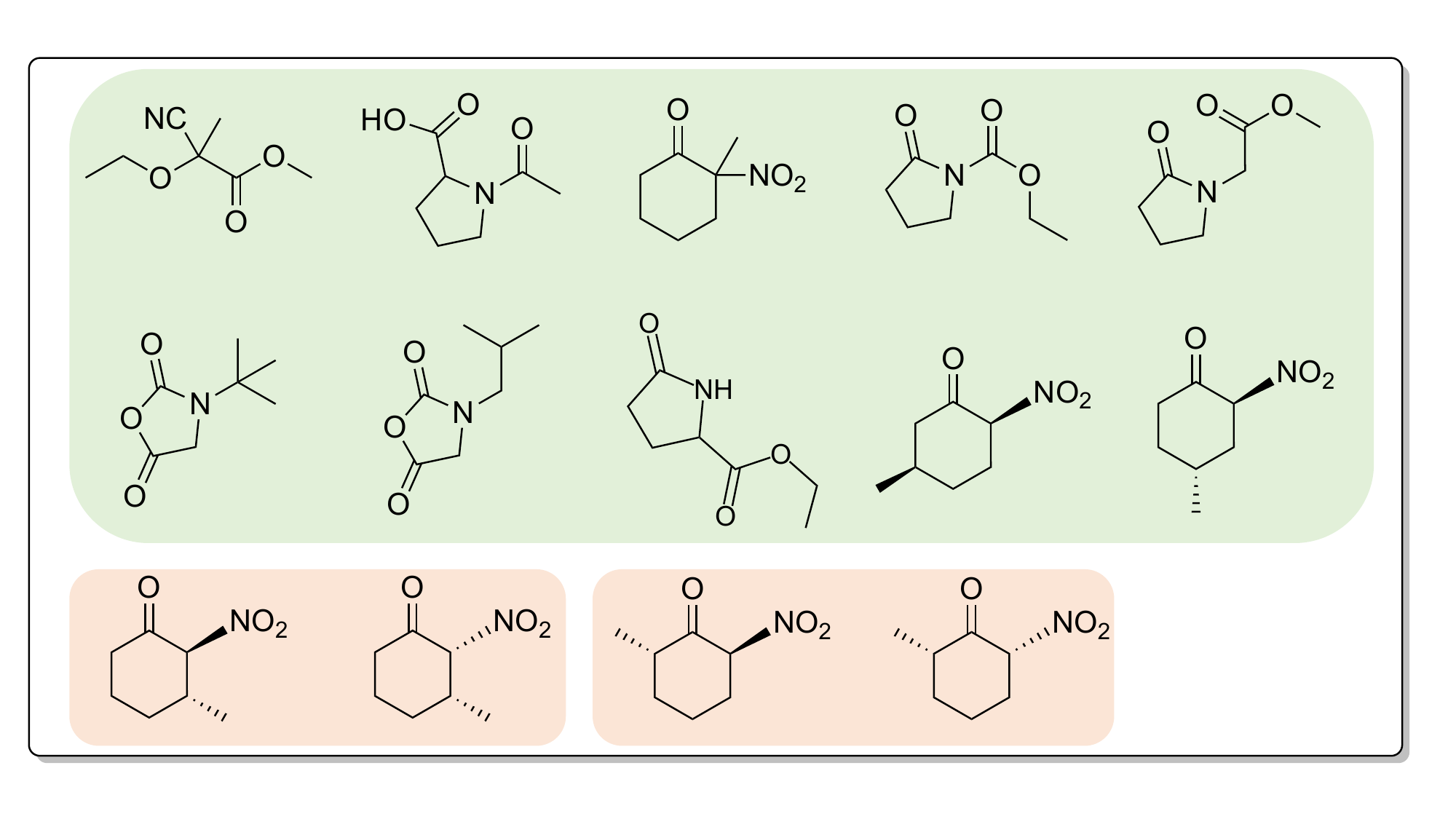}
    \caption{Isomer demo for $\text{C}_{7}\text{H}_{11}\text{NO}_{3}$}
    \label{fig:isomer-demo}
\end{figure}

\section{Additional Results on Peak Assignment}

In complex natural product molecules, it is a common situation that the local contents of some atoms within the same molecule exhibit a high degree of similarity. It gives rise to challenges for the atomic alignment, as some atoms correspond to ppm values in close proximity. However, our K-M\textsuperscript{3}AID model is capable of recognizing each of the atoms with effective learnt embeddings and deciphering the correspondences among the atoms and the peaks at zero-shot. Two complex natural product molecules with multiple rings (4 and 4, respectively) and multiple chiral centers (6 and 8, respectively) are taken to showcase the effectiveness of atomic alignment (see Appendix Figure~\ref{fig:atom-alignment-specific}).

\begin{figure}[!ht]
    \centering
    \includegraphics[width=1\textwidth]{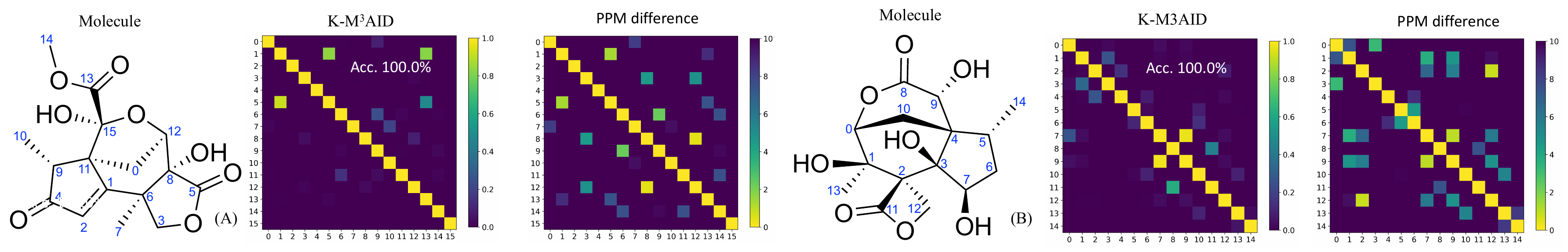}
    \caption{Examples of Zero-shot Atomic Alignment for Complex Natural Products. Yellow cells in the PPM difference represent the ground truth alignment.}
    \label{fig:atom-alignment-specific}
    \vspace{0pt} 
\end{figure}

In  molecular A in Figure ~\ref{fig:molecule-compare-appendix}, atom 13 and atom 14 are tertiary carbons (attaching to 3 carbons and 1 hydrogen) and on the same 5-member ring, corresponding to the ppm of 34.3 and 35.6, respectively. The similar local content of these two atoms fools SP and WP. In addition, WP fails with more atomic alignments. The molecular B is chemical symmetric regarding atom 0. Thus, atom 1 and atom 3 correspond to the same peak on the spectra. The ppm of atom 1 and atom 3 is 114.2, the ppm of atom 2 and atom 4 is 110.0. While there is 4.2 difference, SP and WP fails to pick up right alignment for atom 1 and atom 3. In contrast, K-ID succeed to align the atoms with peaks in both molecules.

\begin{figure}[h]
    \centering
    \includegraphics[width=1\textwidth]{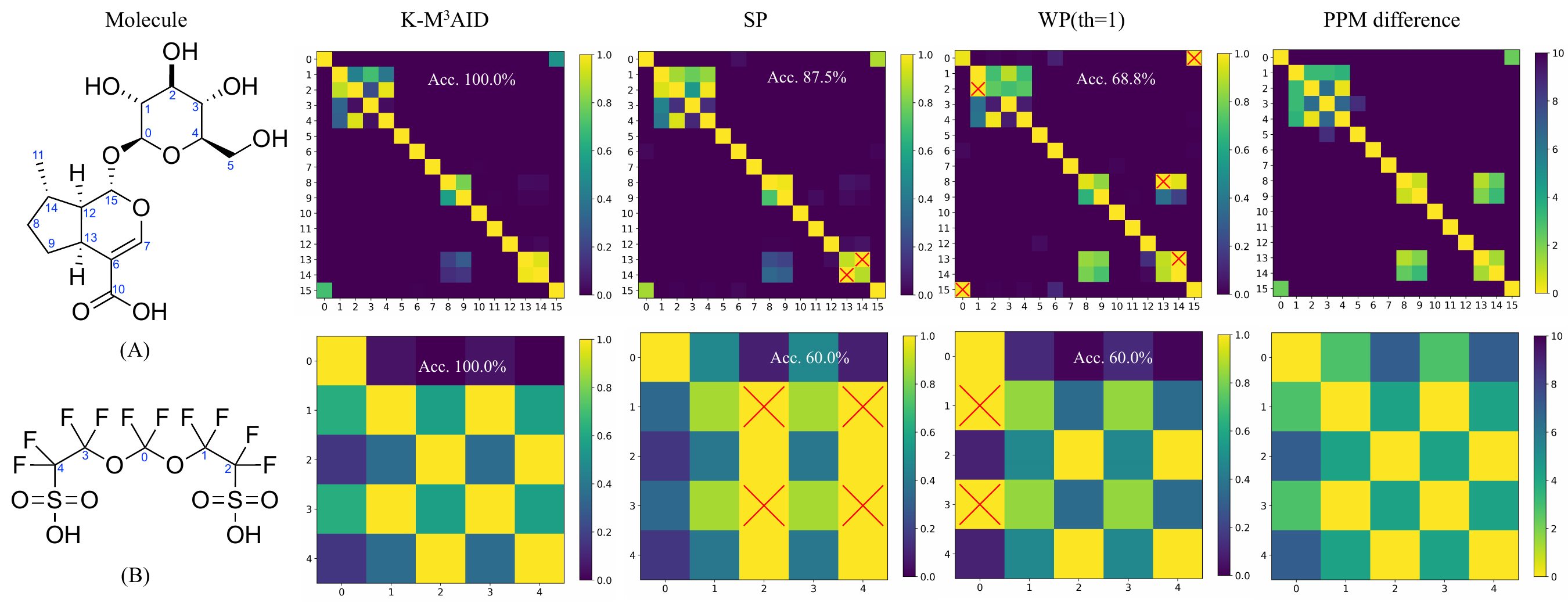}
    \caption{Extra case studies of IE-Meta-MMA.Yellow cells in the PPM differerence represent the ground truth alignment, and red cross represents the wrong alignment.}
    \label{fig:molecule-compare-appendix}
\end{figure}

\end{document}